\documentclass{article}

\usepackage{microtype}
\usepackage{graphicx}
\usepackage{subfigure}
\usepackage{booktabs} %

\usepackage{hyperref}

\usepackage[accepted]{icml2024}

\usepackage{amsmath}
\usepackage{amssymb}
\usepackage{mathtools}
\usepackage{amsthm}

\usepackage[capitalize,noabbrev]{cleveref}

\theoremstyle{plain}
\newtheorem{theorem}{Theorem}[section]

\newtheorem{lemma}[theorem]{Lemma}
\newtheorem{corollary}[theorem]{Corollary}
\theoremstyle{definition}
\newtheorem{definition}[theorem]{Definition}

\theoremstyle{remark}

\usepackage{amsmath,amsfonts,bm}

\def\eqref#1{equation~\ref{#1}}

\def\ceil#1{\lceil #1 \rceil}

\def\1{\bm{1}}

\def\eps{{\epsilon}}

\DeclareMathAlphabet{\mathsfit}{\encodingdefault}{\sfdefault}{m}{sl}
\SetMathAlphabet{\mathsfit}{bold}{\encodingdefault}{\sfdefault}{bx}{n}

\newcommand{\E}{\mathbb{E}}

\newcommand{\R}{\mathbb{R}}

\DeclareMathOperator*{\argmax}{arg\,max}
\DeclareMathOperator*{\argmin}{arg\,min}

\DeclareMathOperator{\sign}{sign}

\usepackage[utf8]{inputenc} %
\usepackage[T1]{fontenc}    %
\usepackage{hyperref}       %
\usepackage{url}            %
\usepackage{booktabs}       %
\usepackage{amsfonts}       %
\usepackage{nicefrac}       %
\usepackage{microtype}      %
\usepackage{xcolor}         %
\usepackage{graphicx}
\usepackage{subfigure}
\usepackage{booktabs} %

\usepackage{amsmath}
\usepackage{amssymb}
\usepackage{mathtools}
\usepackage{amsthm}
\usepackage{adjustbox}

\usepackage[disable,textsize=tiny]{todonotes}

\usepackage{paper}

\newcommand{\mypara}[1]{\paragraph{#1}}
\renewcommand{\hh}{\hat{h}}

\icmltitlerunning{Two Heads are Actually Better than One}

\begin{document}
\twocolumn[
\icmltitle{Two Heads are \textit{Actually} Better than One:\\ Towards Better Adversarial Robustness via\\ Transduction and Rejection}

\icmlsetsymbol{equal}{*}
\begin{icmlauthorlist}
\icmlauthor{Nils Palumbo}{equal,wisc}
\icmlauthor{Yang Guo}{equal,wisc}
\icmlauthor{Xi Wu}{google}
\icmlauthor{Jiefeng Chen}{wisc}
\icmlauthor{Yingyu Liang}{wisc}
\icmlauthor{Somesh Jha}{wisc}
\end{icmlauthorlist}

\icmlaffiliation{wisc}{Depart of Computer Sciences, University of Wisconsin-Madison, Madison, WI, USA}
\icmlaffiliation{google}{Google}

\icmlcorrespondingauthor{Nils Palumbo}{npalumbo@wisc.edu}
\icmlcorrespondingauthor{Yang Guo}{yguo@cs.wisc.edu}

\icmlkeywords{Machine Learning, ICML}

\vskip 0.3in
]

\printAffiliationsAndNotice{\icmlEqualContribution} %

\begin{abstract}
Both transduction and rejection have emerged as important techniques for defending against adversarial perturbations. A recent work by \citep{goldwasser2020beyond} showed that rejection combined with transduction can give \emph{provable} guarantees (for certain problems) that cannot be achieved otherwise. Nevertheless, under recent strong adversarial attacks (GMSA~\citep{chen2022towards}), Goldwasser et al.'s work was shown to have low performance in a practical deep-learning setting.  In this paper, we take a step towards realizing the promise of transduction+rejection in more realistic scenarios. Our key observation is that a novel application of a reduction technique in~\citep{tramer2022detecting}, which was until now only used to demonstrate the vulnerability of certain defenses, can be used to actually construct effective defenses. Theoretically, we show that a careful application of this technique in the transductive setting can give significantly improved sample-complexity for robust generalization. Our theory guides us to design a new transductive algorithm for learning a selective model; extensive experiments using state of the art attacks (AutoAttack, GMSA) show that our approach provides significantly better robust accuracy (81.6\% on CIFAR-10 and 57.9\% on CIFAR-100 under $l_\infty$ with budget 8/255) than existing techniques~\citep{croce2020robustbench}. The implementation is available at \url{https://github.com/nilspalumbo/transduction-rejection}.
\end{abstract}

\section{Introduction}

A recent line of research~\citep{goldwasser2020beyond, montasser2021transductive, goodfellow2019research, wang2021fighting, pmlr-v119-wu20f} has investigated augmenting models with \emph{transduction} (leveraging unlabeled test input to revise the learned model) and \emph{rejection} (allowing a model to reject on certain input) to defend against adversarial perturbations. There are in general two classes of algorithms. One class is \emph{transduction-only}. For example, \citep{montasser2021transductive} showed that robust learning with transduction allows for significant improvements in sample complexity, reducing dependency on VC dimension from exponential to linear; however, this comes at the cost of significantly greater assumptions on the data ($\OPT_{\U^2}$ for the realizable case rather than the $\OPT_\U$ of the inductive setting
\footnote{The optimal robust risk is $\OPT_\U = \inf_{h \in \calH} \Pr_{(x,y) \sim \calD} \left[\exists z \in \U(x) : h(z) \neq y\right]$. For $\U$ which are perturbations up to $\epsilon$ in some metric, $\U^x$ is a perturbation of up to $x\epsilon$, see Section~\ref{sec:preliminaries} for more details.}).

The other class is to have both transduction and rejection. For example, \citep{goldwasser2020beyond} studied this setting and showed even more surprising results, not achievable with transduction or rejection alone. However, one prominent limitation of these works seems to be that none has yet resulted in practical robust learning mechanisms in the deep learning setting typically considered.

In this paper, we take a step towards realizing the promise of transduction+rejection in more realistic scenarios. Compared to \citep{goldwasser2020beyond}, which considers arbitrary perturbations, we focus on the classic and practical scenario of bounded perturbations for deep learning. Somewhat surprisingly, we show that a novel application of Tramèr’s classifier-to-detector technique~\citep{tramer2022detecting}, which has thus far only been applied to indicate that certain defenses are vulnerable, in the transductive setting can give significantly improved sample-complexity for robust generalization, noting that bounded perturbations are critical for the construction to work. To obtain these improvements, we do not require stronger assumptions on the data, as with \citep{montasser2021transductive}; in the realizable case, we only need to assume $\OPT_{\U^{2/3}} = 0$, which is even better than the $\OPT_{\U} = 0$ assumption in the inductive case.

\begin{table*}[ht]\label{tbl:generalization-bounds}
\caption{\small \textbf{Summary of generalization bounds for the four settings}.
Compared to transduction alone and \citep{goldwasser2020beyond}, our defense weakens the necessary conditions in the realizable case and improves the asymptotic error in the agnostic case. Compared to induction and rejection alone, sample complexity has a linear rather than exponential dependence on the VC dimension. Compared to \citep{goldwasser2020beyond}, the dependence on the error bound $\epsilon$ improves from inverse quadratic to inverse linear in the realizable case. Note that \citep{goldwasser2020beyond} requires the existence of a hypothesis with bounded error on the perturbed data in the agnostic case, and hence does not tolerate all possible perturbations.}
  \vskip 0.1in
\begin{adjustbox}{width=\textwidth,center}
\begin{tabular}{l ll l l}
\toprule
                             & \multicolumn{3}{c }{Realizable}                                                                                                                                                 & \multirow{3}*{Agnostic Generalization Bound}                                                                                                                                                        \\
                            \cline{2-4}
                             & Soundness  & Completeness  & Generalization                                                          \\ 
                             &  Condition &  Condition &  Bound                                                                                                \\ \midrule
Induction~\citep{montasser2019vc}  & OPT$_{\U} = 0$ & OPT$_{\U} = 0$ &$\calO\left(\frac{2^{\VC(\calH)} \log(n) +  \log(1/\delta)}{n}\right)$  & OPT$_{\U} + \calO\left(\sqrt{\frac{2^{\VC\left(\calH\right)} +\log(1/\delta)}{n}}\right)$ 
                              \\ \hline
Transduction~\citep{montasser2021transductive} & OPT$_{\U^2} = 0$ & OPT$_{\U^2} = 0$ & $\calO\left(\frac{\VC(\calH)\log(n)+ \log (1/\delta) }{n}\right)$ & 
 $2 \mathrm{OPT}_{\calU^2}+O\left(\sqrt{\frac{\VC(\calH)+\log (1 / \delta)}{n}}\right)$                             \\ \hline
Rejection (Theorem~\ref{thm:sample-complexity-ind-real}, \ref{thm:sample-complexity-ind-agn}) & OPT$^{\rej}_{\U} = 0$ & OPT$^{\rej}_{\U} = 0$ & $ \calO\left(\frac{2^{\VC(T(\calH))} \log(n) +  \log(1/\delta)}{n}\right)$  & OPT$^{\rej}_{\U} + \calO\left(\sqrt{\frac{2^{\VC\left(T(\calH)\right)} +\log(1/\delta)}{n}}\right)$ \\ \hline
Transduction{+}Rejection~\citep{goldwasser2020beyond} & $\OPT_{\U} = 0$ & $\OPT_{\U} = 0$ & $\calO\left(\sqrt{\frac{\VC(\calH)\log(n)}{n}}+ \frac{\log (1/\delta) }{n}\right)$  &  $2\OPT_{\U} + 2\sqrt{2\OPT_{\II}} + \calO\left(\sqrt{\frac{\VC(\calH)\log n + \log(1/\delta)}{n}}\right)$ \\ 
Transduction{+}Rejection  (Theorem~\ref{thm:rejection-simplified-realizable}, \ref{thm:transduction-agnostic}) & OPT$_{\U^{2/3}} = 0$ & OPT$_{\U^2} = 0$ & $\calO\left(\frac{\VC(\calH)\log(n)+ \log (1/\delta) }{n}\right)$  &  $2 \mathrm{OPT}_{\U^{2/3}}+O\left(\sqrt{\frac{\VC(\calH)+\log (1 / \delta)}{n}}\right)$ \\ 
\bottomrule
\end{tabular}
\end{adjustbox}
\label{tab:bounds}
\vskip -0.1in
\end{table*}

Our theory guides us to identify a practical transductive algorithm for learning a robust selective model. As a component, we present a simple empirical approximation to the reduction which enables the computationally efficient realization of the improvement to robustness offered by rejection; our experiments show that the the robustness of models utilizing our rejection-only defense very closely matches the theoretical bound (i.e. the robustness achievable to adversarial budget $\epsilon/2$). While our approach does not have the theoretical guarantees of the computationally inefficient construction, it is a significant step towards developing an efficient reduction, left as an open problem by~\citep{tramer2022detecting}.

In addition, we present an objective for general adaptive attacks targeting selective classifiers based on our algorithm. Our transductive defense algorithm gives strong empirical performance on image classification tasks, both against our adaptive attack and against existing state-of-the-art attacks such as AutoAttack and standard GMSA. On CIFAR-10, we obtain 81.6\% transductive robust accuracy with rejection, a significant improvement on the current state-of-the-art result of 71.1\%~\citep{peng2023robust, croce2020robustbench} for robust accuracy up to the perturbation considered ($l_\infty$ with budget $\epsilon = 8/255$); on CIFAR-100, we obtain 57.9\% transductive robust accuracy with rejection, significantly exeeding the strongest existing baseline of 42.7\%~\citep{wang2023better, croce2020robustbench} with the same adversarial budget. 
 
The rest of the paper is organized as follows. Section~\ref{sec:related} reviews main related work,
and Section~\ref{sec:preliminaries} presents some necessary background. We develop our theory results
in Section~\ref{sec:theory}. Guided by our theory, Section~\ref{sec:method} develops a practical
robust learning algorithm, leveraging both transduction and rejection. We provide systematic experiments
in Section~\ref{sec:experiment}, and conclude in Section~\ref{sec:conclusion}.

\section{Related Work} \label{sec:related}

In recent years, there have been extensive studies on adversarial robustness in the traditional inductive learning setting, where the model is fixed during the evaluation phase~\citep{carlini2017towards, goodfellow2015explaining,moosavi2016deepfool}. Most popular and effective methods are adversarial training, such as PGD~\citep{madry2018towards}, TRADES~\citep{zhang2019theoretically}. These methods are effective against adversaries on small dataset like MNIST, but still ineffective on complex dataset like CIFAR-10 or ImageNet~\citep{croce2020robustbench}. Defenses beyond adversarial training have been proposed but most are broken by strong adaptive attacks~\citep{croce2020reliable, tramer2020adaptive}.

To break this robust bottleneck, recent work has proposed alternative settings with relaxed yet realistic assumptions, particularly by allowing rejection and transduction. In robust learning with rejection (a.k.a., abstain), we allow rejection of adversarial examples instead of correctly classifying all of them~\citep{tramer2022detecting}. Variants of adversarial training with rejection option have been considered~\citep{laidlaw2019playing,pang2022two, chen2021revisiting,kato2020atro, sotgiu2020deep, he2022your}, including generalizations to unseen attacks~\citep{stutz2020confidence} and to certified robustness~\citep{sheikholeslami2020provably, baharlouei2022improving, sheikholeslami2022denoised}. \citep{tramer2022detecting} proves an equivalence between robust learning with rejection and standard robust learning in the inductive setting and shows that the evaluation of past defenses with rejection was unreliable.

The other approach is to define an alternative notion of adversarial robustness via transductive learning, i.e. "dynamically" ensuring robustness on the particular given test samples rather than on the whole distribution. Similar settings have been studied but under the view of "test-time defense" or "dynamic defense"~\citep{goodfellow2019research, wang2021fighting, pmlr-v119-wu20f}. \citep{goldwasser2020beyond} is the first paper to formalize transductive learning for robust learning, and the first to consider transduction+rejection. 
It considers general adversaries on test data and presents novel theoretical guarantees.
\citep{chen2022towards} formally defines the notion of transductive robustness as a maximin problem and presents a principled adaptive attack, GMSA. \citep{montasser2021transductive} discusses robust transductive learning against bounded perturbation from a learning theory perspective and obtains corresponding sample complexity.

\section{Preliminaries} \label{sec:preliminaries}

\begin{table*}[h]
    \caption{Summary of the robust error in all settings. Note that transductive error of the learner $\A$ is the corresponding notion of error where $h = \A(\bx, \by, \tbz)$.} 
    \vskip 0.1in
    \begin{adjustbox}{width=\textwidth,center}
    \begin{tabular}{l l l}
        \toprule
        & Robust Error & Robust Error (with Rejection) \\ 
        \midrule
        Inductive & $\operatorname{err}_{\U}(h; x, y) := \sup_{z \in \U(x)} \mathbb{1}\{h(z) \neq y\}$
        & $\operatorname{err}^{\rej}_{\U}(h; x, y) := \sup_{z \in \calU(x)} \mathbb{1}\{h(z) \notin \{y, \perp\} \lor h(x) \neq y\}$
        \\
        \hline
        Transductive&
        $\operatorname{err}_{\U}(h; \bx, \by, \tilde{\bz}, \tilde{\by}) := \frac{1}{m} \sum_{i=1}^m  \mathbb{1}\left\{ h\left(\tilde{z}_i\right) \neq \tilde{y}_i \right\}$
        &  $\operatorname{err}_{\U}^{\rej}(h; \bx, \by, \tilde{\bx}, \tilde{\bz}, \tilde{\by})
  :=  
\frac{1}{m} \sum_{i=1}^{m} \mathbb{1}\left\{ 
    \begin{array}{l}
        \left( h\left(\tilde{z}_{i}\right) \notin \{\tilde{y_i} \} \land \tilde{z}_{i} = \tilde{x}_{i}\right) \\
        \lor \left( 
            h\left(\tilde{z}_{i}\right) \notin \{\tilde{y_i}, \perp\} \land \tilde{z}_{i} \neq \tilde{x}_{i}  
        \right)
    \end{array}
\right\}$ 
        \\
    \bottomrule
    \end{tabular}
    \end{adjustbox}
    \label{tab:errs}
    \vskip -0.1in
\end{table*}

Let $\calX$ denote the input space, $\calY$ the label space, $\calD$ the clean data distribution over $\calX \times \calY$. We will assume binary classification for our theoretical analysis: $\calY = \{\pm 1\}$. Let $\calU(x)$ denote the set of possible perturbations of an input $x$, e.g., for $\ell_p$ norm perturbation of budget $\epsilon$, $\U$ is the $\ell_p$ ball of radius $\epsilon$: $\U(x) = \{z: \|z - x\|_p \le \epsilon\}$. We assume $\U$ satisfies $\forall x\in \calX, x \in \U(x)$; essentially all interesting perturbations satisfy this. Let $\U^2(x) := \{z: \exists t \in \U(x), \textrm{such that~} z \in \U(t) \}$, and $\U^{-1}(x) := \{z: x \in \U(z)\}$. If a perturbation set $\Lambda$ satisfies $\Lambda^2 = \U$, then we say $\Lambda = \U^{1/2}$; $\U^{-1/2} = (\U^{-1})^{1/2}$. When $\U$ is the $\ell_p$ ball of radius $\epsilon$, $\U^2$ is that of radius $2\epsilon$, $\U^{-1} = \U$, and $\U^{1/2}$ is that of radius $\epsilon/2$; we define $\UC$, $\UT$, and $\UTI$ similarly. 

All learners are provided with $n$ i.i.d. training samples
\footnote{Here $\bx=(x_i)_{i=1}^n$ and similarly with $\by, \tilde{\bx}, \tilde{\by}$, etc. We will also overload the notation $\U$, e.g., $\U(\bx) := \{ \boldsymbol{u} \in \calX^n: u_i \in \calU(x_i) \}$.}
$(\bx, \by) = (x_i, y_i)_{i=1}^n \sim \calD^n$. There are $m$ i.i.d. test samples $(\tilde{\bx}, \tilde{\by}) \sim \calD^m$, and the adversary can perturb $\tilde{\bx}$ to $\tilde{\bz} \in \calU(\tilde{\bx})$.

 We describe the main settings below; the corresponding notions of error are in Table~\ref{tab:errs}. For each setting, we define risk as the expected worst-case error up to the perturbation $\U$, and empirical risk similarly.

\mypara{Induction.} In the traditional robust classification setting (e.g.,~\citep{madry2018towards}); also called the inductive setting or simply induction), the learning algorithm (the defender) is given training set $(\bx, \by)$, learns a classifier $h: \calX \mapsto \calY$ from some hypothesis class $\calH$.

\mypara{Rejection.}
In the setting of robust classification with rejection, the classifier has the extra power of abstaining (i.e., outputting a rejection option denoted by $\perp$), and furthermore, rejecting a perturbed input does not incur an error. The learning algorithm is given training set $(\bx, \by)$ and learns a 
\emph{selective classifier}, defined as a function 
\begin{equation}\label{eqn:selective_classifier}
   h: \calX \mapsto  \calY \cup \{\perp\} 
\end{equation}
from some hypothesis class $\calH$ which, given a sample $x$, either outputs a label $y \in \calY$ or abstains from prediction with an output of $\perp$.
An error occurs only when $h$ rejects a clean input, or accepts and misclassifies. We define additionally $\operatorname{OPT}^{\rej}_{\U}:= \inf_{h\in \calH}\R^{\rej}_{\U}(h; \calD)$.

\mypara{Transduction.}
In the setting of robust classification with transduction (e.g.,~\citep{montasser2021transductive}), the learning algorithm (the transductive learner) has access to the unlabeled test input data; the goal is to predict labels only for these given test inputs (a transductive learner need not generalize).
The learner $\mathbb{A}$ is given the training data $(\bx, \by)$ and the (potentially perturbed) test inputs $\tilde{\bz}$, and outputs $m$ labels $h(\tilde{\bz})=(h(\tilde{z}_i))_{i=1}^m$ as predictions for $\tilde{\bz}$. That is, the learner is a mapping $\mathbb{A}: (\calX \times \calY)^n \times \calX^m \mapsto \calY^m$.
A special case is when $\mathbb{A}$ learns a classifier $h$ and use it to label $\tilde{\bz}$; the labels are also denoted as $h(\tilde{\bz})$.

\mypara{Our setting: Transduction+Rejection.}
A transductive learner for selective classifiers $\mathbb{A}$ is given $(\bx, \by, \tilde{\bz})$, and outputs rejection or a label for each input in $\tilde{\bz}$. 
That is, the learner is a mapping $\mathbb{A}: (\calX \times \calY)^n \times \calX^m \mapsto (\calY \cup \{\perp\})^m$.  
An error occurs when it rejects a clean test input or accepts and misclassifies. Hence, we present the appropriate notion of error in Table~\ref{tab:errs}, the natural extension of the rejection-only error to the transductive setting, with the key difference being that we penalize rejection only if the sample is not perturbed (as transductive learners produce outputs only on the provided test data, there is no notion of rejecting $\tilde{x}_i$ if it has been perturbed).

\section{Theoretical Analysis} \label{sec:theory}

In this section, we present theoretical results which guide the design of our algorithm (see Section~\ref{sec:method}). We show that, by applying an algorithm which produces a hypothesis robust with reduced $\epsilon$ about an intermediate perturbation and incorporating rejection (via Tram\`er's classifer-to-detector reduction~\yrcite{tramer2022detecting}), we can derive an algorithm with strong guarantees up to the full $\epsilon$; in particular, the full algorithm obtains linear dependence on the VC dimension with greatly reduced necessary conditions compared to transduction alone. This suggests that this simple approach may provide significant improvements to robustness. We find in that this is indeed the case: our algorithm, described in Section~\ref{sec:method}, obtains significantly improved robustness compared to existing baselines; see Section~\ref{sec:experiment} for more details.

We focus on the realizable case for the setting with transduction+rejection here, for more details and results for the agnostic case and the setting with rejection alone see Appendix~\ref{app:proof}. For comparison with existing results in the inductive-only and transduction-only settings~\citep{montasser2019vc,montasser2021transductive}, we follow their setup: assume there exists a classifier (without rejection) with 0 robust error from a hypothesis class $\calH$ of VC-dimension $\VC(\calH)$; the goal is to design a learner with a small robust error. 

\begin{theorem}\label{thm:rejection-simplified-realizable}
For any $n \in \mathbb{N}$, $\delta > 0$, hypothesis class $\calH$ of classifiers without rejection, perturbation set $\CalU$ such that $\U= \U^{-1}$ and $\U^{1/3}$ exists, and distribution $\CalD$ over $\CalX \times \CalY$ satisfying $\mathrm{OPT}_{\U^{2/3}} = 0$, 
there exists a transductive learner $\mathbb{A}$ that constructs a set of selective classifiers (of the form \cref{eqn:selective_classifier}) $\Delta$ s.t. the following is true: 
with probability $\ge 1 - \delta$ over $(\bx, \by) \sim \mathcal{D}^{n}$, $(\tilde{\bx}, \tilde{\by}) \sim \mathcal{D}^{n}$, we have that for any $\tbz \in \U(\tbx)$, if $\Delta \neq \emptyset$, then for any $h \in \Delta$,
\footnote{Note that $\Delta$ is a function of $\bx$, $\by$, and $\tbz$, so this is more precisely a bound of $\sup_{\tbz \in \U(\tbx), h \in \mathbb{A}(\bx, \by, \tbz)}\operatorname{err}^{\rej}_{\U}(h; \bx, \by, \tilde{\bx}, \tbz, \tilde{\by})$.}
\begin{align*}
 \operatorname{err}^{\rej}_{\U}(h; \bx, \by, \tilde{\bx}, \tbz, \tilde{\by}) \leq \frac{\VC(\calH)\log(2n)+ \log(1/\delta)}{n}.
\end{align*}
\end{theorem}

For $\U$ satisfying our conditions (including $l_p$ balls), we obtain a stronger guarantee than those using only transduction or only rejection. First, compared to the guarantee for transduction without rejection~\citep{montasser2021transductive} (see Table~\ref{tab:bounds}), our result requires weaker assumptions on the data: we need $\textrm{OPT}_{\U^{2/3}} = 0$ rather than $\textrm{OPT}_{\U^2} = 0$.
For example, consider the $\ell_p$ norm perturbation: $\U(x) = \{z: \|z - x\|_p \le \epsilon\}$. Transduction alone requires that there exists a classifier with 0 robust error up to the perturbation $\U^2$, i.e. up to an $\ell_p$ norm perturbation of adversarial budget $2\epsilon$. In contrast, our result shows that using both transduction and rejection only requires there exists a classifier with 0 robust error up to perturbation $\UTS$, corresponding to adversarial budget of $2\epsilon/3$. Equivalently, for a data distribution with a margin $2\epsilon$, transduction without rejection can only handle adversarial perturbations with budget $\epsilon$, while combining transduction and rejection can handle adversarial perturbations with budget $3\epsilon$, tolerating three times the adversarial magnitude. Second, compared to rejection only (see Table~\ref{tab:bounds}), this bound has a linear sample complexity rather than exponential. Therefore, combining transduction and rejection has the benefits of both techniques. 

Furthermore, note that the result bounds the rate of incorrect rejections as well, i.e. the rate of rejections on clean data, with the same bound as a direct consequence of the definition of robust error under transduction and rejection. However, the result, while potentially very strong, comes with the caveat that the defense is not guaranteed to find a nonempty $\Delta$ (i.e., the defense is sound but may not be complete) under conditions weaker than $\OPT_{\U^2} = 0$; by Lemma~\ref{lem:rejection-agnostic-delta} in Appendix~\ref{app:trans-real}, $\Delta$ is guaranteed to be nonempty, and hence we have completeness, under the same conditions as transduction alone. Hence, the result is strictly stronger than the result for transduction alone~\citep{montasser2021transductive}.

Consider an adversarial budget $\epsilon$, and suppose $\tz$ is the given potentially perturbed test input and $\tx$ is the corresponding clean test input.
To obtain the guarantee, we need to find a model which is $\epsilon/3$-robust at $q = \tx + (\tz-\tx)/3$. Such a model always exists when $\mathrm{OPT}_{\U^{2/3}}=0$. 
However, given only $\tz$ without knowing $q$ or $\tx$, our algorithm finds a model $\epsilon/3$-robust at every perturbation within $2\epsilon/3$ of $\tz$ and thus $\Delta$ may be empty.

While weaker conditions don't guarantee that we find a model satisfying the conditions, the result still provides intuition for the success of our derived empirical defense. For typical data distributions and hypothesis classes, it might be expected that, if we fail to find a $\epsilon$-robust hypothesis at the fully-perturbed data, we will nevertheless be more likely to find a model which is robust nearer the clean data distribution (i.e. where the condition is required by the theory) rather than further away. Determining conditions for this is an interesting direction for future research.

Such conditions do exist: in Appendix~\ref{app:trans-real} we present a distribution $\D$, hypothesis class $\calH$, and perturbation $\U$ for which $\Delta$ is guaranteed to be nonempty and the error bound above applies, but where trasduction has a minimum asymptotic error of 1/2.

\mypara{Proof Sketch.}
For intuition, think of $\U$ as the $\ell_p$ norm perturbation with adversarial budget $\epsilon$. We omit technical details; see Appendix~\ref{app:trans-real} for the complete proof.
Consider some clean training set $\bx, \by$, clean test set $\tbx, \tby$, with perturbed test data $\tbz$ with $\tbz_i$ within $\epsilon$ of $\tbx_i$. 
Let $\tbz' = \tbx + (\tbz-\tbx)/3$ be the intermediate perturbation a third of the way between $\tbx$ and $\tbz$.

First, following \citep{montasser2021transductive}, define the set of robust hypotheses $\Delta_{\calH}^{\UT}(\bx, \by, \tbz')$ as  
$\Delta_{\calH}^{\UT}(\bx, \by, \tbz')  = \left\{
    \operatorname{R}_{\UT}(h;\bx, \by) = 0 \wedge
    \operatorname{R}_{\UT}(h;\tbz') = 0
    \right\}$
where
$\operatorname{R}_{\U}(h;\bz,\by) = \sup_{\tbx \in \U(\bz)}{1\over n}\sum_{i=1}^n \mathbb{1}\{h(\tx_i) \neq y_i\}$ and $\operatorname{R}_{\U}(h;\bz) = \operatorname{R}_{\U}(h;\bz, h(\bz))$.

That is, we find those classifiers that satisfy: (1) they are $\epsilon/3$-robustly correct (i.e., correct and robust to perturbations of budget $\epsilon/3$) on the training data $(\bx, \by)$; (2) they have $\epsilon/3$ margin on the intermediate perturbations $\tbz'$ (i.e., have the same prediction for all perturbations of budget $\epsilon/3$).
This then guarantees, as shown in \citep{montasser2021transductive}, that with high probability, for any $h \in \Delta_{\calH}^{\UT}(\bx, \by, \tbz')$ the robust error facing perturbation of budget $\epsilon/3$ is bounded by $\frac{\VC(\calH)\log(2n)+ \log(1/\delta) }{n}$ if $\OPT_{\U^{2/3}} = 0$.

Next, following \citep{tramer2022detecting}, define a transformation $\FUT$ that maps a classifier without rejection, $h$, to the selective classifier (see Equation~\ref{eqn:selective_classifier}) $\hh = \FUT(h)$:
\begin{equation}
\hh(x) = \begin{cases}
h(x) & \text{if } \forall  x' \in \UTI(x)\,, h(x') = h(x)\\
\perp & \text{otherwise}
\end{cases}.
\end{equation}
That is, $\hh$ rejects $x$ if it is within $\epsilon/3$ from $h$'s decision boundary, otherwise accepts and predicts $h(x)$.

\begin{figure}[!t]
\centering
\includegraphics[width=\columnwidth]{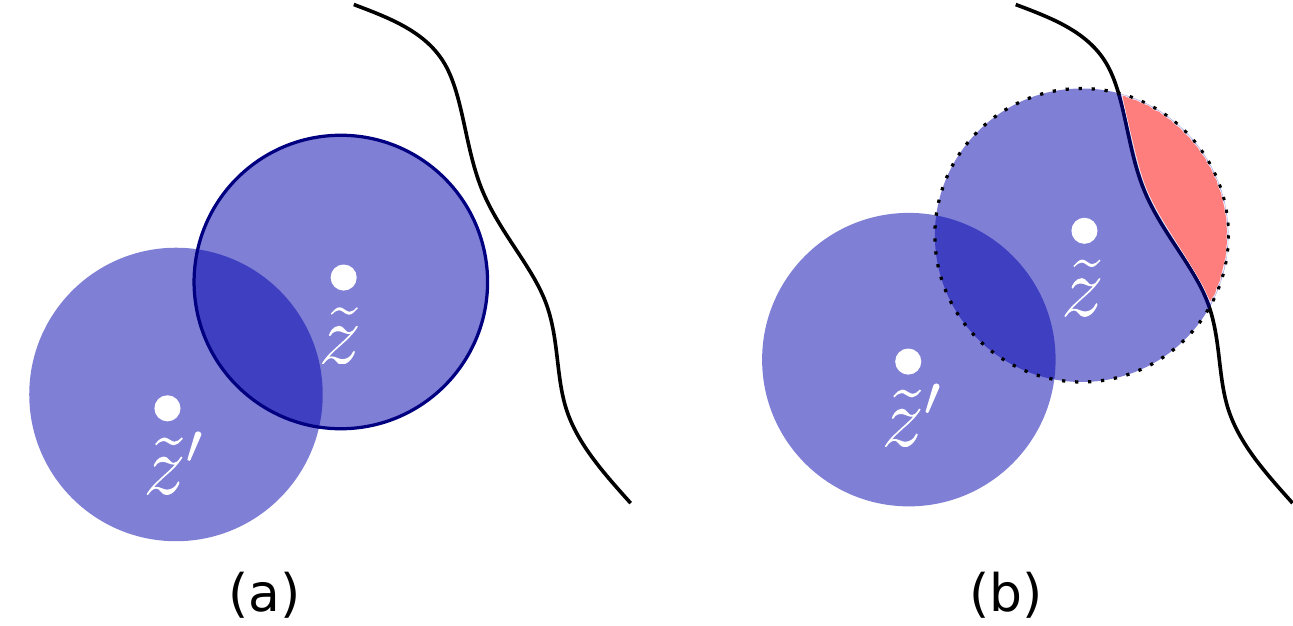}
\vspace{-2mm}
\caption{(a) $h$ is $\epsilon/3$-robust at $\tz$; $\hh$ correctly classifies $\tz$.\\
(b) $h$ is not $\epsilon/3$-robust at $\tz$; $\hh$ rejects $\tz$.} 
\label{fig:transductive-realizable}
\end{figure}

Now, consider a clean test sample $(\tx,\ty)$ and $\tx$'s adversarial perturbation $\tz$. Define an intermediate perturbation $\tz' = \tx + (\tz-\tx)/3$. We will show that if $h$ is correct at $\tz'$, then $\hh$ makes no error at $\tz$.

If $\tz = \tx$, then $\tz'= \tx = \tz$. Since $h$ is $\epsilon/3$-robust at $\tz'$, $h(\tz) = h(\tz') = \ty$ and so $\hh(\tz) = \ty$ which is correct.
Otherwise, we need to consider two cases: 
\textbf{(a)} $h$ is $\epsilon/3$-robust at $\tz$; \textbf{(b)} $h$ is not. See visualization in Figure~\ref{fig:transductive-realizable}. In both cases, the $\epsilon/3$-balls about $\tz$ and $\tz'$ intersect. Let $\tz''$ be some point in the intersection. Since $h$ is $\epsilon/3$-robust at $\tz'$, $h(\tz'') = h(\tz') = \ty$. Now, in case (a) where $h$ is $\epsilon/3$-robust at $\tz$, $h(\tz) = h(\tz'') = \ty$, which is correct. In case (b) where $h$ is not $\epsilon/3$-robust at $\tz$, $\hh$ rejects $\tz$ and makes no error.

Hence if $h$ is correct at $\tz'$, then $\hh$ makes no error at $\tz$. %
So the error bound for $h$ implies the desired error bound for any $\hh$ in the set
\begin{equation}
    \Delta' = \left\{\hh = \FUT(h): h \in \Delta_{\calH}^{\UT}(\bx, \by, \tbz') \right\}.
\end{equation}

As we have access only to the adversarial test data $\tbz$, 
to ensure $\epsilon/3$-robustness at the unknown $\tbz'$, we need to ensure $\epsilon$-robustness at $\tbz$.
Let 
\begin{equation}
  \Delta''
  :=
  \cup \left\{\hh = \FUT(h): h \in \bigcap_{\tbz' \in \UTSI(\tbz)}\Delta_{\calH}^{\UT}(\bx, \by, \tbz') \right\}
\end{equation}
and let
$\hat{\Delta} = \bigcup_{\tbz' \in \UTSI(\tbz)}\Delta_{\calH}^{\UT}(\bx, \by, \tbz').$
By the above, as $\Delta'' \subseteq \Delta'$, any $\hh$ in $\Delta''$ achieves the desired bound.
If $|\hat{\Delta}| = 1$, then $|\Delta'| = 1$ and as $\Delta' \subseteq \hat{\Delta}$, $\hat{\Delta} = \Delta'$ and so any $\hh$ in $\Delta'' \cup \hat{\Delta}$ likewise achieves the bound.

Hence, if we let 
\begin{equation}
    \Delta = \begin{cases}\Delta'' \cup \hat{\Delta} & |\hat{\Delta}| = 1\text{, }\\\Delta'' & \text{ otherwise }\end{cases},
\end{equation} we obtain the theorem statement.

\section{Defense by Transduction and Rejection} \label{sec:method}

The analysis of \Cref{thm:rejection-simplified-realizable} suggests the following defense algorithm: (1) first obtain a classifier $h$ that is robust and correct on the training data and also robust on the test inputs, (2) then transform $h$ to a selective classifier $\hh$ by rejecting inputs too close to the decision boundary of $h$. We describe the resulting defense below:

\textbf{Step (1)} To get $h$, we perform adversarial training on both the training set and the test set, using a robust cross-entropy objective. As in TADV~\citep{chen2022towards} we train with private randomness.
Specifically, we train a model with softmax output as the class prediction probabilities $h^{s}$ and the class prediction is $h(x) = \argmax_{y \in \calY} h^s_y(x)$. 
Given the labeled training data $(\bx, \by)$ and the test inputs $\tilde{\bz}$, we optimize the following objective:  

\begin{equation} \label{eqn:transductive-loss}
  \min_h \left(\begin{array}{l}
    {1\over n}\sum_{(x,y) \in (\bx, \by)}\left[\begin{array}{l}
        \LCE(h^{s}(x), y) \\
        + \max_{x'\in \U(x)} \LCE\left(h^{s}(x'), y\right)
       \end{array}\right] \\
    + {\lambda\over m}\sum_{\tilde{z} \in \tilde{\bz} } \left[ \max_{\tilde{z}'\in \U(\tilde{z})} \LCE\left(h^{s}(\tilde{z}'), h(\tilde{z})\right) \right]
   \end{array}\right)
\end{equation}

where $\LCE$ is the cross-entropy loss and $\lambda>0$ is a hyper-parameter.  

\textbf{Step (2)} Having learned $h$, we now turn $h$ into a selective classifier $\hh$. To do this, we need to compute the transformation $\FUT$, however, the construction following~\citep{tramer2022detecting} is computationally inefficient and as such is not practical for an empirical defense. Hence, we present a simple but effective approach which performs similarly in practice, as we show in Section~\ref{sec:rej-exps}.

\textbf{Empirical Classifier to Selective Classifier Transformation:} Recall that $\hh$ rejects the input $x$ if there exists $x' \in \U^{1/3}(x)$ with $h(x) \neq h(x')$; otherwise accepts and predicts the label $h(x)$.
So we only need to determine the existence of $x' \in \U^{1/3}(x)$ with $h(x) \neq h(x')$. 

We use a standard inductive attack, PGD, for this by solving:
\begin{align}
   \argmax_{ x' \in \U^{1/3}(x) } \quad \LCE(h^{s}(x'), h(x)). 
\end{align}
When $\U$ is $\ell_p$ norm ball of radius $\epsilon$, the constraint is then $\|x' - x\| \le \epsilon/3$. In practice, we can generalize this to $\|x' - x\| \le \epsilon_{\text{defense}}$ where 
$\epsilon_{\text{defense}}$ is a hyper-parameter we call the \emph{rejection radius}.

Taken together, we obtain a strong defense with transduction and rejection which significantly outputerforming existing baselines (see Section~\ref{sec:tldr-robustness}), which we refer to as \textbf{TLDR} (Transductive Learning Defense with Rejection).

\textbf{Discussion on Computational Cost.} The computational cost of training with TLDR is higher than that of standard adversarial training, in particular, by a factor of at most two; the cost of the transformation $\FUT$ is the same as that of PGD. As with general transductive defenses, the training process must be repeated for each new batch of samples; hence TLDR is suited to applications with minimal latency requirements, which may amortize the cost of training over a large batch of test samples, on the order of the full training set.

\textbf{Discussion on Evaluation.} Adversarial evaluations of novel defenses are well known to be challenging~\citep{chen2022towards, binarization}; hence, we construct an adaptive attack targeting our defense in Section~\ref{sec:attack} and thorougly evaluate it in Section~\ref{sec:experiment}. As we incorporate GMSA~\citep{chen2022towards} in our attack, we must perform multiple iterations of training in evaluation, each of which is computationally costly. Hence, attacking and evaluating TLDR is extremely computationally expensive.

\subsection{Adaptive Attacks} \label{sec:attack}

Since no strong adaptive attacks exist for the new transduction+rejection setting to our knowledge, we design one here.
Our attack is based on GMSA in \citep{chen2022towards}, which has been shown to be a strong attack for transductive defense (without rejection). 

The goal of the attack is to find perturbations $\tilde{\bz}$ of the clean test inputs $\tilde{\bx}$ such that the transductive learner has a large error when given $(\bx,\by,\tilde{\bz})$. GMSA runs in stages; in each stage $t$, it simulates the transductive learner on the current data set $(\bx,\by,\tilde{\bz}_t)$ to get a classifier $h_t$, and then maximizes the minimum or average loss of $\{h_i\}_{i=1}^t$ to get the updated perturbations of the test inputs $\tilde{\bz}_{t+1}$ (called $\text{GMSA}_\text{MIN}$ and $\text{GMSA}_\text{AVG}$, respectively). See \citep{chen2022towards} for the details.

GMSA does not directly apply to our setting since we have selective classifiers $\hh$ with a rejection option which is not considered in GMSA. 
Our contribution is to design a method to get the updated perturbations $\tilde{\bz}$ of the test inputs in each stage such that  the selective classifier incurs a large error.
Recall that $\hh$ constructed from $h$ incurs error in two cases: (1) it accepts $\tilde{z}$ and misclassifies with $h(\tilde{z}) \neq y$; (2) $\tilde{z} = \tilde{x}$ and it rejects $\tilde{z}$. We consider the two cases below.

\textbf{Case (1)} We will propose a novel loss  measuring the loss of a selective classifier $\hh$ on a perturbation $(\tilde{z},y)$ from a clean test point $(\tilde{x}, y)$ for such kind of error; maximizing this loss gives the desired $\tilde{z}$.
Recall that we need $\tilde{z}$ to be accepted and also the prediction $h(\tilde{z}) \neq y$. For the latter, we can maximize $\LCE(h^s(\tx), y)$ where $h^s$ is the class probabilities of $h$ (i.e., its softmax output). The former is equivalent to 
$
  \min_{h(\tilde{z}') \neq h(\tilde{z})}\|\tilde{z} - \tilde{z}'\| \geq \epsilon_{\text{defense}}.
$

Now, suppose $\LDB(\tilde{z}')$ is a \textit{surrogate loss} function on the closeness to the decision boundary; it increases when $\tilde{z}'$ gets closer to the decision boundary of $h$. Then the condition is equivalent to
$\left\|\tilde{z} - p(\tilde{z})\right\| = \epsilon_{\text{defense}}$ where
$p(\tilde{z}) = \argmax_{\|\tilde{z}' - \tilde{z}\| \leq \epsilon_{\text{defense}}} \LDB(\tilde{z}')$.
Now, as the maximum value of $\left\|\tilde{z} - p(\tilde{z})\right\|$ is exactly $\epsilon_{\text{defense}}$, we would like to maximize $\left\|\tilde{z} - p(\tilde{z})\right\|$ to satisfy the condition. 

Summing up, for this case, we would like to maximize:
\begin{equation}\label{eqn:lrej}
\begin{aligned}
    \LREJ(\tilde{z}, y) & := \LCE(h^s(\tilde{z} ), y) +
    \lambda' \left\|\tilde{z} - p(\tilde{z})\right\|,\\
    \textrm{~where~}
     p(\tilde{z}) & = \argmax_{\|\tilde{z}' - \tilde{z}\| \leq \epsilon_{\text{defense}}} \LDB(\tilde{z}')
\end{aligned}
\end{equation}
and $\lambda'>0$ is a hyper-parameter. 
Finally, for $\LDB$, the following definition works well in our experiments: 
$
  \LDB(\tilde{z}') := \text{rank}_2 \ h^s(\tilde{z}') - \max h^s(\tilde{z}'), 
$
which is maximized at the decision boundary as the top-two class probabilities are equal. 

\textbf{Case (2)} A critical step in an effective application of $\LREJ$ to a transductive attack is the selection of which points to perturb. To do this, we apply a post-processing step after finding $\tilde{z}$ by maximizing~(\eqref{eqn:lrej}). We must predict whether $\hh$ is more likely to incur error on $\tilde{z}$ or on the clean input $\tilde{x}$ (i.e., $\hh(\tilde{x}) \neq y$). If we expect that the clean point is likely to be incorrectly classified or rejected, then we update $\tilde{z}$ to $\tilde{x}$. In GMSA, we have access to a series of models trained on previous attack iterations; we estimate the likelihood of success at $\tz$ and $\tx$ by the fraction of previous models which fail at each point.

Summing up the two cases and combining with GMSA gives our final attack (details in Algorithm~\ref{alg:rej-attack-gmsa} in Appendix~\ref{app:attack-details}).

\section{Experiments}\label{sec:experiment}

This section performs experiments to evaluate the proposed method TLDR and compare it with baseline methods (e.g., those using only rejection or transduction). 
Our main findings are: {\bf 1)} TLDR outperforms the baselines significantly in robustness, confirming the advantage of combining transduction and rejection. {\bf 2)} Our adaptive attack is significantly stronger than existing attacks which were not designed for the new setting, providing a strong evaluation. 
{\bf 3)} Rejection rates rise steadily with the rejection radius, but few clean samples are rejected and the robust accuracy remains stable.

\subsection{Datasets and Defense/Attack Setup}
We evaluate on MNIST~\citep{lecun1998mnist} and CIFAR-10~\citep{krizhevsky2009learning}. We consider an adversarial budget of $\epsilon=0.3$ in $l_\infty$ on MNIST and $\epsilon=8/255$ in $l_\infty$ on CIFAR-10. 
For defense, on MNIST, we use a LeNet architecture; on CIFAR-10 we use a ResNet-20 architecture. In both cases, we train for 40 epochs with a learning rate of 0.001 using ADAM for optimization. On MNIST, we use 40 iterations of PGD during training with a step size of 0.01. On CIFAR-10, we use 10 iterations of PGD in training with a step size of 2/255. In training TLDR, we set $\lambda = 0.176$ after a warm start period in which $\lambda = 0$. We use a rejection radius of $\epsilon/4$ for selective classifiers.
For attack, we use 10 iterations of GMSA on both datasets. On MNIST, we use 200 steps of PGD with a stepsize of 0.01 while generating adversarial examples. On CIFAR-10, the PGD attacks use 100 steps with a stepsize of 1/255. Defense settings used while training models in GMSA (including internal PGD settings) are the standard defense settings. Internal optimizations in the calculation of $\LREJ$ use 10 steps of PGD with a stepsize of 15\% of the rejection radius. We use $\lambda' = 1$ in $\LREJ$; we observe little sensitivity to the parameter.

\subsection{Attack Evaluation}\label{sec:attack-evaluation}

\begin{table}[!t]
      \caption{Robust accuracy by different attacks on TLDR. The strongest attack is  \textbf{boldfaced}.}
    \begin{center}
        \begin{tabular}{l c c}
    \toprule
     Attack &  MNIST &  CIFAR-10 
    \\
    \midrule
    PGD ($\LCE$)  & 0.991 &  0.794 \\
    PGD ($\LREJ$)  & 0.988 &  0.781 \\
    AutoAttack  & 0.989 &  0.756 \\
    GMSA ($\LCE$)   & 0.988 &   0.853\\
    \textbf{GMSA ($\LREJ$)} &   \textbf{0.972} &   \textbf{0.739}\\
    \bottomrule
    \end{tabular}  
\label{tab:ablation-method}
    \end{center}
    \caption{Robust accuracy under different attack losses on a fixed adversarially trained model with rejection, AutoAttack for comparison.
The strongest attack is \textbf{boldfaced}.} 
    \vskip 0.1in
  \begin{adjustbox}{width=\columnwidth,center}
    \begin{tabular}{l c c}
    \toprule
    Loss & MNIST &  CIFAR-10 
    \\
    \midrule
    AutoAttack~\citep{croce2020reliable} & 0.980 & 0.592\\
    \midrule
    $\LCE$  & 0.977   & 0.524\\ 
    $\LREJ(\LCE)$   & 0.974   & 0.470\\
    \textbf{$\LREJ$} &   \textbf{0.973}   & \textbf{0.458}\\
    \bottomrule
    \end{tabular} 
    \end{adjustbox}
    \label{tab:ablation-lrej}
    \vskip -0.1in
\end{table}

\begin{table*}[ht]
    \caption{Results on MNIST and CIFAR-10. Robust accuracy is 1 - robust error; see \Cref{sec:preliminaries}. $\PREJ$ is the percentage of inputs rejected. The baseline results are from~\citep{chen2022towards}. The strongest attack against each defense is shown. 
    The best result is  \textbf{boldfaced}.} 
    \vskip 0.1in
    \begin{adjustbox}{width=\textwidth,center}
    \begin{tabular}{l  l l cc cc}
    \toprule
    \multirow{2}{*}{Setting} & \multirow{2}{*}{Defense} & \multirow{2}{*}{Attacker} & \multicolumn{2}{c }{MNIST} & \multicolumn{2}{c}{CIFAR-10}\\ \cline{4-5} \cline{6-7}
    & & & $\PREJ$ & Robust accuracy & $\PREJ$ & Robust accuracy
    \\
    \midrule
    Induction & AT~\citep{madry2018towards} & AutoAttack & -- & 0.897 & -- & 0.448\\
    \midrule
    Rejection only & AT (with rejection) & PGD ($\LREJ$) & 0.852 & 0.968 & 0.384 & 0.634\\
    \midrule 
    \multirow{2}{*}{Transduction only} & RMC~\citep{pmlr-v119-wu20f} & GMSA ($\LCE$) & -- & 0.588 & -- & 0.396\\
    & DANN~\citep{domainadversarial} & GMSA ($\LCE$) & -- & 0.062 & -- & 0.055\\
    & TADV~\citep{chen2022towards} & GMSA ($\LCE$) & -- & 0.943 & -- & 0.541 \\
    \midrule
    Transduction+Rejection & URejectron~\citep{goldwasser2020beyond} & GMSA ($\LDISC$) & 0.274 & 0.721 & \textbf{0.000} & 0.145\\
    \bottomrule
    Transduction+Rejection & \textbf{TLDR (ours)}  & GMSA ($\LREJ$) & \textbf{0.126} & \textbf{0.972} & 0.208 & \textbf{0.739}\\
    \bottomrule
    \end{tabular} 
    \end{adjustbox}
    \label{tab:result}
    \vskip -0.1in
\end{table*}

\begin{table*}[ht]
    \caption{Comparison with state-of-the-art~\citep{peng2023robust, wang2023better, croce2020robustbench} on CIFAR-10 and CIFAR-100 under $l_\infty$ perturbations with budget 8/255.
    The best result is  \textbf{boldfaced}.} 
     \vskip 0.1in
    \begin{adjustbox}{width=\textwidth,center}
    \begin{tabular}{l l l l cc cc}
    \toprule
    \multirow{2}{*}{Setting} & \multirow{2}{*}{Defense} & \multirow{2}{*}{Architecture} & \multirow{2}{*}{Attacker} & \multicolumn{2}{c }{CIFAR-10} & \multicolumn{2}{c}{CIFAR-100}\\ \cline{5-6} \cline{7-8}
    & & & &  $\PREJ$ & Robust accuracy & $\PREJ$ & Robust accuracy
    \\
    \midrule
    Induction & \citep{peng2023robust} & Ra WideResNet-28-10 & AutoAttack & -- & 0.651 & -- & 0.372\\
    Induction & \citep{peng2023robust} & Ra WideResNet-70-16 & AutoAttack & -- & 0.711 & -- & 0.388\\
    Induction & \citep{wang2023better} & WideResNet-28-10 & AutoAttack & -- & 0.673 & -- & 0.388\\
    Induction & \citep{wang2023better} & WideResNet-70-16 & AutoAttack & -- & 0.707 & -- & 0.427\\
    \midrule
    Transduction+Rejection & \textbf{TLDR (ours)} & ResNet-20 & GMSA ($\LREJ$) & 0.208 & 0.739 & -- & -- \\
    Transduction+Rejection & \textbf{TLDR (ours)} & WideResNet-28-10 & GMSA ($\LREJ$) & 0.111 & \textbf{0.816} & 0.171 & \textbf{0.579} \\
    \bottomrule
    \end{tabular} 
    \end{adjustbox}
    \label{tab:sota}
     \vskip -0.1in
\end{table*}

\Cref{tab:ablation-method} shows the results of different attack methods on TLDR. Previous work~\citep{chen2022towards} shows that transduction-aware attacks are necessary against transductive defenses; we observe that attacks (PGD on $\LCE$ or $\LREJ$ and AutoAttack) from the traditional setting perform poorly against our defense. We can also see that GMSA significantly outperforms even a rejection-aware transfer attack (referred to as PGD targeting $\LREJ$; note that PGD and AutoAttack do \textit{not} target the final model in this case, given the transductive setting, but instead target a proxy trained by the adversary); see Algorithm~\ref{alg:rej-attack-transfer} in Appendix~\ref{app:attack-details} for the full details.

This shows that GMSA is critical for attacking a transductive defender; while PGD and AutoAttack are strong against an inductive model, they performs poorly facing transduction. Finally, we observe that GMSA with $\LCE$ is much weaker than GMSA with $\LREJ$. This shows another key component in our adaptive attack, the loss $\LREJ$, is also critical to get a strong attack against our defense.

To further investigate the importance of $\LREJ$, we attack an adversarially trained model with rejection with PGD on different losses: $\LREJ$, cross-entropy $\LCE$, and $\LREJ$ with $\LDB$ replaced by $\LCE$, with AutoAttack given for comparison. Table~\ref{tab:ablation-lrej} shows that $\LREJ$ significantly outperforms both PGD targeting alternative losses and AutoAttack. See Appendix~\ref{app:attack-ablation} for an evaluation of the effectiveness with which $\LREJ$ targets rejection using the binarization test~\citep{binarization}.

\subsection{Robustness of TLDR}\label{sec:tldr-robustness}

\paragraph{Baselines.}
(1) AT: adversarial training~\citep{madry2018towards}; (2) AT (with rejection): adversarial training (AT) with rejection; (3) RMC~\citep{pmlr-v119-wu20f}; (4) DANN~\citep{domainadversarial}; (5) TADV~\citep{chen2022towards}; (6) Rejectron~\citep{goldwasser2020beyond}. Among them, (1) is in the traditional induction setting, (2) is rejection only, (3)(4)(5) are transduction only, and (6) incorporates both transduction and rejection. 

\paragraph{Evaluation.}
We attack the defenses and report the robust accuracy (1 - the robust error defined in Section~\ref{sec:preliminaries}).
To attack inductive classifiers, we use AutoAttack~\citep{croce2020reliable}. For inductive selective classifiers, we use PGD on the rejection-aware loss $\LREJ$ from Eqn (\ref{eqn:lrej}). For transductive classifiers, we use GMSA which has been shown to be a strong adaptive attack on transduction~\citep{chen2022towards}. Finally, for our transductive selective classifiers, we use our adaptive attack in~\Cref{sec:attack} (roughly GMSA with $\LREJ$). For Rejectron~\citep{goldwasser2020beyond} we use GMSA with a loss function $\LDISC$ targeting their defense; see Appendix~\ref{app:rejectron-experiments} for the details. We include an ablation of the two core components of TLDR (the transductive loss term and the transformation into a selective classifier) in Appendix~\ref{app:attack-ablation}.

For transductive models, we report the stronger of $\text{GMSA}_\text{MIN}$ and $\text{GMSA}_\text{AVG}$.
Inductive models are trained with standard adversarial training~\citep{goodfellow2015explaining}, 
 and transductive models with the TLDR loss in Eqn (\ref{eqn:transductive-loss}). As Rejectron depends heavily on a key hyperparameter determining confidence needed to reject, we report the results for the parameter value strongest against our attack. The best-performing value on CIFAR-10 effectively eliminated the possibility of rejection (hence the rejection rate of 0); other choices resulted in near-0 robust accuracy.

\begin{table*}[!t]
     \caption{Comparison of our rejection-only defense with budget $\epsilon$ to induction-only defenses with budget $\epsilon / 2$.}
     \vskip 0.1in
  {\small
  \begin{center}
    \begin{tabular}{l l l l l l c}
      \toprule
      Dataset & Model & Defense & Attacker & $\epsilon$ (training) & $\epsilon$ (attack) & Robust accuracy\\
      \midrule
      CIFAR-10 & Resnet-20 & AT~\citep{madry2018towards} & PGD ($\LCE$) & 8/255 & 8/255 & 0.478\\
      CIFAR-10 & Resnet-20 & AT~\citep{madry2018towards} & PGD ($\LCE$) & 4/255 & 4/255 & 0.564\\
      CIFAR-10 & Resnet-20 & \textbf{AT (with rejection) [ours]} & PGD ($\LREJ$) & 4/255 & 8/255 & 0.564\\
      \midrule
      CIFAR-10 & WideResnet-28-10 & AT~\citep{madry2018towards} & PGD ($\LCE$) & 8/255 & 8/255 & 0.429\\
      CIFAR-10 & WideResnet-28-10 & AT~\citep{madry2018towards} & PGD ($\LCE$) & 4/255 & 4/255 & 0.602\\
      CIFAR-10 & WideResNet-28-10 & \textbf{AT (with rejection) [ours]} & PGD ($\LREJ$) & 4/255 & 8/255 & 0.601\\
      \midrule
      CIFAR-100 & WideResNet-28-10 & AT~\citep{madry2018towards} & PGD ($\LCE$) & 8/255 & 8/255 & 0.181\\
      CIFAR-100 & WideResnet-28-10 & AT~\citep{madry2018towards} & PGD ($\LCE$) & 4/255 & 4/255 & 0.307\\
      CIFAR-100 & WideResNet-28-10 & \textbf{AT (with rejection) [ours]} & PGD ($\LREJ$) & 4/255 & 8/255 & 0.307\\
      \bottomrule 
     \end{tabular}
     \label{tab:rejection-results}
  \end{center}
  }
    \vskip -0.1in
\end{table*} 

\textbf{Comparison of Defenses.}
\Cref{tab:result} shows the robust accuracy and rejection rate of different methods. We observe that either transduction or rejection can improve the performance, while combining both techniques leads to the best results.
In particular, our defense outperforms existing transductive defenses such as RMC and DANN. Results for $l_2$ perturbations are given for $l_2$ in Appendix~\ref{app:attack-ablation}.
See Table~\ref{tab:sota} for a comparison to the state-of-the-art. With a much smaller ResNet-20 architecture, TLDR outperforms the strongest existing baseline on CIFAR-10, and, with a WideResNet-28-10 architecture, we obtain an in robust accuracy; on CIFAR-100 of over 10\%, we obtain an improvement in robust accuracy of over 15\%.

\paragraph{Discussion on Evaluation.} As our key focus is on demonstrating the potential advantages of one setting (transduction+rejection) over others, comparisons between settings are necessary. In each setting, robust accuracy represents the same concept, the fraction of samples on which we are correct. The difference between settings lies in their different notions of “correctness”; each concept of correctness incorporates both the potential advantages and the disadvantages of each setting, e.g. in the rejection case, a new type of error is possible: rejecting a clean sample. Hence, we compare the fraction of samples on which we can be correct between settings (and between defenses in the same setting).
 
\subsection{Rejection-Only Defense}\label{sec:rej-exps}

\citep{tramer2022detecting} shows that the existence of a classifier with $x$ robust accuracy with adversarial budget $\eps$ implies the existence of a selective classifier with $x$ robust accuracy with adversarial budget $2\eps$; however, as the construction used is computationally inefficient, this has not yet been realized in practice. Table~\ref{tab:rejection-results} evaluates our rejection-only defense by comparing its results on the full adversarial budget (8/255) to the theoretical bound obtained by a classifier with the half-budget of 4/255. In each case, our empirical transformation results in a robust accuracy is very close to the results obtained by Tram\`er's idealized computationally inefficient approach. In this way, our approach enables practical realization of Tram\`er's upper bound on gains from rejection in the inductive case.

\section{Conclusion}\label{sec:conclusion}

Existing works on leveraging transduction and rejection gave mixed results on their
benefits for adversarial robustness. In this work we take a step in realizing their promise
in practical deep learning settings. Theoretically, we show that a novel application of
Tram\`er's results give improved sample complexity for robust learning in the bounded
perturbations setting. Guided by our theory, we identified a practical robust learning algorithm
leveraging both transduction and rejection. Systematic experiments confirm the benefits of our
constructions. There are many future avenues to explore, such as improving the theoretical bounds,
and improving the efficiency of our algorithms.

\section*{Impact Statement}

The rapid advance of ML methods in recent years has coincided with increasing deployment in safety-critical applications. Hence, the potential societal risks associated with unreliable models, in particular those which can easily be misled by adding imperceptible noise, has increased significantly. Our work presents a principled yet practical adversarial defense, helping to limit these risks.

\newpage 

\appendix
\onecolumn

\begin{center}
  \textbf{\LARGE Appendix}
\end{center}

\section{Proof Details}\label{app:proof}

Before introducing the proof for the generalization results, we first need to make some additional definitions. We define the \textit{empirical robust risk} as 
\[
\mathrm{\hat{R}}_{\mathcal{U}}(h ; S)= \sum_{(x,y) \in S}\left[\sup _{z \in \mathcal{U}(x)} \mathbb{1}\{h(z) \neq y\}\right]
\]
And we can define the \textit{empirical robust risk under rejection} accordingly: 
\[
\hat{\mathrm{R}}_{\mathcal{U}}^{\rej}(h ; S)= \sum_{(x,y)\in S}\left[\sup _{z \in \mathcal{U}(x)} \mathbb{1}\{h(x) \neq y \lor h(z) \notin \{y, \perp\}\}\right]
\]
And we can define the corresponding robust empirical risk minimization procedure (under rejection) as follows: 

\[
\textrm{RERM}_{\calH}(S):= \underset{h \in \calH}{\operatorname{argmin}}\; \hat{\mathrm{R}}_{\mathcal{U}}(h ; S)
\]

\[
\textrm{RERM}_{\calH}^{\rej}(S):= \underset{h \in \calH}{\operatorname{argmin}} \;  \hat{\mathrm{R}}_{\calU}^{\rej}(h ; S)
\]

\subsection{Rejection Only: Realizable Case}\label{app:ind-real}
\begin{definition}[Realizable Robust PAC Learnability under Rejection] For $\calY = \{0, 1\}$, $\forall \epsilon, \delta \in (0,1), \calH = \calH_c \times \calH_r$, the sample complexity of realizable robust $(\epsilon, \delta)$ - PAC learning of $\calH$ with respect adversary $\mathcal{U}$ under rejection, denoted as $\calM_{\textrm{RE}}(\epsilon,\delta; \calH, \mathcal{U})$, is defined as the smallest $m \in \mathbb{N} \cup \{0\}$ for which there exists a learning rule $\calA :  (\calX \times \calY)^m  \longmapsto (\calY \cup \{\perp\})^{\calX}$ s.t. for every data distribution $\calD$ over $(\calX \times \calY)^m$ where there exists a predictor with rejection option $h^* \in \calH$ with 0 risk, $\mathrm{R}_{\calU,\textrm{rej}}(h^*; \calD) = 0$ with probability at least $1-\delta$ over $S \sim \calD^m$, 
\[
\mathrm{R}_{\calU}^{\rej} (\calA(S); \calD) \le \epsilon 
\]
If no such $m$ exists, $\mathcal{M}_{\textrm{RE}}(\epsilon, \delta; \calH, \calU) = \infty$. We say that $\calH$ is robustly PAC learnable under rejection in the realizable setting with respect to adversary $\calU$ if $\forall \epsilon, \delta \in (0,1)$, $\calM_{\textrm{RE}}(\epsilon, \delta; \calH, \calU)$ is finite. 
\end{definition}

\begin{theorem}[Sample Complexity for Realizable Robust PAC Learning under Rejection]\label{thm:sample-complexity-ind-real}
In the realizable setting, for any $\calH = \calH_c \times \calH_r$ and $\mathcal{U}$, and any $ \epsilon, \delta \in (0, 1/2)$, 
\begin{align}
    \mathcal{M}_{\textrm{RE}}(\epsilon, \delta; \calH, \mathcal{U}) =2^{O\left(\left(d_r+d_c\right)\log \left(d_r + d_c\right) \right)} \frac{1}{\varepsilon} \log \left(\frac{1}{\varepsilon}\right)+O\left(\frac{1}{\varepsilon} \log \left(\frac{1}{\delta}\right)\right)
\end{align}
where $d_r = \VC(\calH_r), d_c = \VC(\calH_c)$.

The idea of the proof is to adapt the classical sample compression argument \citep{littlestone1986relating} with improvements based on \citep{montasser2019vc, hanneke2019sample, moran2016sample}. The generalization result in the inductive case %
directly comes from \Cref{eqn: inductive-realizable-generalization}.
\end{theorem}

\begin{proof}
First, we define the concept of \textit{sample compression scheme} and \textit{sample compression algorithm}.
\begin{definition}[Sample Compression Scheme]\label{def:sample-compress}
Given $\forall m \in \mathbb{N}$ samples, $S \sim \calD^m$, a \textit{sample compression scheme} of size $k$ is defined by the following pair of functions: 
\begin{enumerate}
    \item Compression function $\kappa: (\calX \times \calY)^m \mapsto (\calX \times \calY)^{\leq k}$. 
    \item Reconstruction function: $\rho: (\calX \times \calY)^{\leq k} \mapsto \calH$.
\end{enumerate}
An algorithm $\calA$ is a \textit{sample compression algorithm} if $\exists \kappa, \rho$ s.t. $\calA(S) = (\kappa \circ \rho)(S)$. 
\end{definition}

Fix $\epsilon, \delta \in (0,1)$, $m > 2(d_r + d_c) \log(d_r + d_c)$. Let the compression parameter, $n =  \calO\left(\left(d_r+d_c\right)\log \left(d_r + d_c\right) \right)$. Let $\calD$ be any distribution, then by realizability of the learner, $\inf _{h \in \calH} \textrm{R}_{\calU}^{\rej}(h ; \calD)=0$. Thus, $\forall S$ sampled from $\calD$, we have $\hat{\textrm{R}}_{\calU}^{\rej} ( \textrm{RERM}_{\calH}^{\rej}(S);S) = 0$.

\paragraph{Compression} First, we define a compression function $\kappa$ as through the following inflation and discretization procedure. Given the training data $S := \{(x_i,y_i)\}_{i\in [m]}$, we define the following index mapping: 
\begin{align}
    I(x) = \min \{i \in [m] : x \in \calU(x_i) \}, \quad \forall x \in \bigcup_{i\in [m]} \calU(x_i).
\end{align}
In another word, this index function outputs the first indexed training sample to include $x$ in its neighborhood. 

Then, we consider the set of RERM mapping learned by a size $n$ subset of the training data: 
\begin{align}
    \hat{\calH}=\{\operatorname{RERM}_{\calH}^{\rej}(L):L \subseteq S,|L|=n\}.
\end{align}
Note that
\begin{align}
    |\hat{\calH}| \leq|\{L: L \subseteq S,|L|=n\}|=\left(\begin{array}{c}m \\ n\end{array}\right)\leq\left(\frac{e m}{n}\right)^{n}.
\end{align}

Then, we inflate the data in the following way: 
\begin{align}
    S_{\mathcal{U}}=\bigcup_{i \in [m]}\left\{\left(x_{I(x)},x, y_{I(x)}\right): x \in \mathcal{U}\left(x_{i}\right)\right\}.
\end{align}
Note that $x_{I(x)}$ can be different from $x_i$. 

Let's define the following transformation $T$: 
\begin{align}\label{eqn:tansform_T}
    T(h)(x,x',y) := \mathbb{1}\{h(x) \neq y \vee h(x') \notin\{y, \perp\}\} ,\; h \in \calH.
\end{align}
And we can obtain the transformed hypothesis class $T(\calH): = \{T(h) | h \in \calH\}$. 

Now, we proceed to define the \textit{dual space} $\calG$ of $T(\calH)$ as the following set of functions. 
\begin{align}
    \calG := \{g_{(x,x',y)} | g_{(x,x',y)}(t) = t(x,x', y), \; t \in T(\calH)\}.
\end{align}
We denote the VC dimension of the dual space as $\VC^*(T(\calH)):= \VC(\calG).$

By Lemma~\ref{lem: VC_rej_loss},
\begin{align}
    \VC(T(\calH)) = \calO \left(\left(d_r + d_c\right)\log\left( d_r + d_c\right)\right).
\end{align}
 
By the classic result in \citep{assouad1983densite}, the VC dimension of the dual space satisfies the following inequality: 
\begin{align}
   VC^*(T(\calH))< 2^{\VC(T(\calH))+1}. 
\end{align}

Now, we can construct the compressed dataset $\hat{S}_{\calU}$ as the following. For each $(x,x',y) \in S_{\calU}$, $\{g_{(x,x',y)}(t) \}_{t \in T(\hat{\calH})}$ gives a labeling. When ranging over $(x,x',y) \in S_{\calU}$, the labeling may not be unique. So for each unique labeling, we choose a representative $(x,x',y) \in S_{\calU}$, and let $\hat{S}_{\calU}$ be the set of the representatives. That is: 
\begin{align}
 \hat{S}_{\calU} = \left\{(x,x',y) \in S_{\calU}  \ \bigg | \ \{g_{(x,x',y)}(t) \}_{t \in T(\hat{\calH})} \text{ provides a unique labeling} \right\}.    
\end{align}

Intuitively, $\hat{S}_{\calU}$ split the infinite size dataset $S_{\calU}$ into finite size according to the labeling of $T(\hat{\calU})$ on the dual space. Thus, $\hat{S}_{\calU}$ is not necessarily unique but always exists. And $|\hat{S}_{\calU} |$ equals the number of possible labeling for $T(\hat{\calH})$. 

Let $d_*:= VC(\calG) = \VC^*(T(\calH))$ denote the VC-dimension of $\calG$, the dual hypothesis class of $T(\hat\calH)$~\citep{assouad1983densite}.  
By applying Sauer's Lemma, we obtain that for $|T(\hat{\calH})| > d_*$, 
\begin{align}
    |\hat{S}_{\mathcal{U}}| \leq \left(\frac{e|T(\hat{\calH})|}{d_*} \right)^{d_*}.
\end{align} 

Let $n = \Theta\left(\VC\left(T\left(\calH\right)\right)\right)$. 
For $m \geq n$, we have
\begin{align}
    |\hat{S}_{\mathcal{U}}| 
    & \leq \left(e |T(\hat{\calH})|  \right)^{d_*}
    \\
    & \leq \left(e|\hat{\calH}  |\right)^{d_*}
    \\ 
    & \leq \left(e \left( \frac{em}{n}\right)^n \right)^{d_*}
    \\ 
    & \leq \left(\frac{e^2 m}{n}\right)^{n d_*}
    \\
    & = \left(\frac{e^2 m}{\VC(T(\calH))} \right)^{\Theta( \VC(T(\calH)) \cdot \VC(T(\calH^*)) )}.
\end{align}

Now we have obtain the compression map: $\kappa(S) = \hat{S}_{\calU}$. 

\paragraph{Reconstruction}
Now, we want to reconstruct a hypothesis from $\hat{S}_{\calU}$. First, suppose we have a data distribution over $\hat{S}_{\calU}$, denoted as $\calP$. This distribution $\calP$ over samples will be later used in the $\alpha-$boosting procedure. 

Then, we sample the set of $n$ i.i.d. samples from $\calP$ and obtain $S' \in \hat{S}_{\calU}$. By classic PAC learning guarantee \citep{blumer1989learnability}, for $n = \Theta(\VC(T(\calH))) = \Theta(d_r + d_c)\log(d_r + d_c)$, we have with non-zero probability $\forall t \in T(\calH)$ with $\sum_{(x,x',y) \in S'}t(x,x',y) = 0$ implies  $\mathbb{E}_{(x,x',y) \sim \calP}t(x,x',y) < 1/9$. 
Let $L = \{(x,y): (x, x', y) \in S'\} \subseteq S$, and $t_{\calP} = T(\textrm{RERM}_{\calH}^{\rej}(L))$. Since $\hat{\textrm{R}}_{\calU}^{\rej}(\textrm{RERM}_{\calH}^{\rej}(L); L) = 0$, $\forall (x,x',y) \in S', t_{\calP}(x,x',y) = 0$. Thus, $\forall \calP$ over $\hat{S}_{\calU}$, there exists a weak learner $ t_{\calP} \in T(\hat{\calH})$, s.t. $\mathbb{E}_{(x,x',y) \sim \calP} \; t_{\calP}(x,x',y) < 1/9$.

Now, we use $t_{\calP}$ as a \textit{weak hypothesis} in a boosting algorithm, specifically $\alpha-$boost algorithm from \citep{schapire2012boosting} with $\hat{S}_{\calU}$ as the dataset and $\calP_k$ generated at each round of the algorithm. Then with appropriate choice of $\alpha$, running $\alpha-$boosting for $K = \calO(\log(|\hat{S}_{\calU}|))$ rounds gives a sequence of hypothesis 
$h_1 , \dots, h_K \in \hat{\calH} $ 
and the corresponding $t_i = T(h_i)$ such that $\forall (x,x',y) \in \hat{S}_{\calU}$,
\begin{align}\label{eqn:alpha-boosting}
& \quad  
\frac{1}{K}\sum_{k =1}^K \mathbb{1}\{h_k(x) \neq y \vee h_k(x') \notin\{y, \perp\}\} 
\\
& = \frac{1}{K}\sum_{k =1}^K t_k(x,x',y) 
\\
& < \frac{2}{9} < \frac{1}{3}.
\end{align}

Since $ \hat{S}_{\calU}$ includes all the unique labellings, $\frac{1}{K}\sum_{k = 1}^K t_k(x,x',y) < \frac{1}{3}, \; \forall (x,x',y) \in \hat{S}_{\calU}$ implies
\begin{align}
    \frac{1}{K}\sum_{k = 1}^K t_k(x,x',y) < \frac{1}{3}, \; \forall (x,x',y) \in S_{\calU}.
\end{align}
Let $\bar{h} :=  \textrm{Majority}(h_1, \dots, h_K)$, i.e., $\bar{h}$ outputs the prediction in $\calY \cup \{\perp\}$ that receives the most votes from $\{h_1, \dots, h_K\}$. Then $\forall (x,x',y) \in \hat{S}_{\calU}$, 
\begin{align}\label{eqn:transform_th}
    \mathbb{1}\{\bar{h}(x) \neq y \vee \bar{h}(x') \notin\{y, \perp\}\} = 0.
\end{align}
This is because: (1) on $x$, less than $1/3$ of $h_i$'s do not output $y$, so $\bar{h}(x)=y$; (2) on $x'$, less than $1/3$ of $h_i$'s do not output $y$ or $\perp$, so the majority vote must be in $y$ or $\perp$, i.e., $\bar{h}(x) \in \{y, \perp\}$. 

In summary, given the same $m$ training samples, we can simply find a $\bar{h}$ with 0 robust error on $S$:

\begin{align}
\hat{\textrm{R}}_{\calU}^{\rej}(\bar{h}; \calD) = \sum_{i=1}^{m}\left[\sup _{z \in \mathcal{U}(x)} \mathbb{1}\{\bar{h}(x) \neq y \lor \bar{h}(z) \notin \{y, \perp\}\} \right]  = 0.    
\end{align}

Now we have the compression set with size: 
\[
nK = \calO(\VC(T(\calH))\log(|\hat{S}_{\calU}|)) = \calO(\VC(T(\calH))^2\VC^*(T(\calH))\log(m/\VC(T(\calH))))
\]
Then, we apply Lemma 11 of \citep{montasser2019vc} (Replacing $\mathrm{R}_{\calU}$ with $\mathrm{R}_{\calU}^{\rej}$ still holds), we obtain for sufficiently large $m$, with probability at least $1-\delta$,  
\begin{align}
    \textrm{R}_{\calU}^{\rej}(\bar{h}; \calD) \leq \calO \left( \VC(T(\calH))^{2} \VC^{*}(T(\calH)) \frac{1}{m} \log (m / \VC(T(\calH))) \log (m)+\frac{1}{m} \log (1 / \delta) \right).
\end{align}

We then can extend the sparsification procedure from \citep{moran2016sample, montasser2019vc} to the rejection scenario. Since $t_1, \dots, t_K \in T(\hat{\calH})$, the classic uniform convergence results~\citep{shalev2014understanding} implies that we can sample $N = \calO(\VC^*(T(\calH)))$ i.i.d. indices $i_1, \dots, i_N \sim \mathrm{Uniform}([K])$ and obtain:
\begin{align}\label{eqn:sparsification-convergence}
\sup_{(x,x',y) \in S_{\calU}} \left| \frac{1}{N}\sum_{j=1}^N t_{i_j}(x,x',y) - \frac{1}{K}\sum_{i=1}^T t_i(x,x',y) \right| < \frac{1}{18}
\end{align}

And thus, we can combine \Cref{eqn:alpha-boosting} with \Cref{eqn:sparsification-convergence} and obtain: 
\[
\forall (x,x',y) \in S_{\calU}, \frac{1}{N}\sum_{j=1}^N t_{i_j}(x,x',y) \leq -\frac{1}{18} + \frac{1}{K}\sum_{i =1}^K t_{k}(x,x',y) <  -\frac{1}{18} + \frac{4}{9} = \frac{1}{2} 
\]

we can further obtain an improved hypothesis $\bar{t}' := \mathrm{Majority}(t_{i_1}, \dots t_{i_N})$ with 
\[
\bar{t}'(x,x',y) = 0, \forall (x,x',y) \in S_{\calU}
\]
Thus, the compression set has a reduced size: 
\[
nN = \calO(\VC(T(\calH))\cdot \VC^*(T(\calH)))
\]
Now, we apply Lemma 11 of \citep{montasser2019vc} and can obtain the following improved bound. Applying similar strategy from \Cref{eqn:transform_th}, we can obtain 
\begin{align}\label{eqn: full-recon}
\bar{h}^{'} := \mathrm{Majority}(h_{i_1}, \dots h_{i_N}) = \rho(\hat{S}_{\calU}) = \calA(S)
\end{align}
which is our full reconstruction map.

Then, for large sample size $m \geq c \; \VC(T(\calH)) \VC^*(T(\calH))$ ($c$ is a sufficiently large constant), with probability at least $1-\delta$, 
\begin{align}\label{eqn: inductive-realizable-generalization}
\textrm{R}_{\calU,\textrm{rej}}(\bar{h}' ; \calD) \leq \calO\left(\VC(T(\calH)) \VC^{*}(\calH) \frac{1}{m} \log (m)+\frac{1}{m} \log (1 / \delta)\right)
\end{align}
Plugging in Lemma~\Cref{lem: VC_rej_loss} and solving for $m$ gives 
\begin{align}
\mathcal{M}_{\textrm{RE}}(\epsilon, \delta; \calH, \mathcal{U}) &=2^{\calO\left(\VC(T(\calH)) \right)} \frac{1}{\varepsilon} \log \left(\frac{1}{\varepsilon}\right)+O\left(\frac{1}{\varepsilon} \log \left(\frac{1}{\delta}\right)\right)    \\
&= 2^{\calO\left(\left(d_r+d_c\right)\log \left(d_r + d_c\right) \right)} \frac{1}{\varepsilon} \log \left(\frac{1}{\varepsilon}\right)+O\left(\frac{1}{\varepsilon} \log \left(\frac{1}{\delta}\right)\right)
\end{align}

\end{proof}

\paragraph{Lemma}[VC dimension of robust loss with rejection]\label{lem: VC_rej_loss}
Let $\VC(\calH_c)= d_c$, and $\VC(\calH_r)= d_r$. Then, $\VC(T(\calH)) = \calO \left(\left(d_r + d_c\right)\log\left( d_r + d_c\right)\right)$.

\begin{proof}
Suppose $d > d_r + d_c$. 

By definition of VC dimension, the max number of labeling of $d$ points is $2^d$ on $h \in T(\calH)$. And since the label of $h$ is a deterministic function of $h_c$ and $h_r$, by Sauer's Lemma, the number of labeling of $h$ is at most $O(d^{d_r}) \times O(d^{d_c}) = O(d^{d_r + d_c})$. 

Thus, $2^d = \mathcal{O}(d^{d_r + d_c})$. And $d = O((d_r+d_c) \log(d_r+d_c))$. 

If $d < d_r + d_c$, $d= \mathcal{O}(d_r + d_c) \log(d_r + d_c)$ by definition. 

\end{proof} 

\subsection{Rejection Only: Agnostic Case}\label{app:ind-agn}

Now, we define notion of PAC learnability in the agnostic case under rejection setting as the follows: 

\begin{definition}[Robust PAC Learnability under Rejection] For $\calY = \{0, 1\}$, $\forall \epsilon, \delta \in (0,1), \calH = \calH_c \times \calH_r$, the sample complexity of robust $(\epsilon, \delta)$ - PAC learning of $\calH$ with respect to perturbation $\mathcal{U}$ under rejection, denoted as $\calM_{\textrm{AG}}(\epsilon,\delta; \calH, \mathcal{U})$, is defined as the smallest $m \in \mathbb{N} \cup \{0\}$ for which there exists a learning rule $\calA : (\calX \times \calY)^m  \longmapsto (\calY \cup \{\perp\})^{\calX}$ s.t. for every data distribution $\calD$ over $(\calX \times \calY)^m$,
\[
\mathrm{R}_{\calU}^{\rej} (\calA(S); \calD) \le \textrm{OPT}_{\U}^{\rej} + \epsilon 
\]
with probability at least $1-\delta$ over $S \sim \calD^m$. If no such $m$ exists, $\mathcal{M}_{\textrm{AG}}(\epsilon, \delta; \calH, \calU) = \infty$. We say that $\calH$ is robustly PAC learnable under rejection if $\mathcal{M}_{\textrm{AG}}(\epsilon, \delta; \calH, \calU)$ is finite for all $\epsilon, \delta \in (0,1)$.
\end{definition}

\begin{lemma}\label{lem:inductive-real-agnostic}
Let $\MRE = \MRE(1/3,1/3;\calH, \calU)$. Then, 
\begin{align}
    \MAG(\epsilon, \delta; \calH, \mathcal{U}) =\calO\left(\frac{\MRE}{\epsilon^2}\log^2\left(\frac{\MRE}{\epsilon}\right)+ \frac{1}{\epsilon^2} \log \left(\frac{1}{\delta}\right) \right)
\end{align}
\end{lemma} 
\begin{proof}
The proof detail follows exactly the same from the Proof of Theorem 8 from \citep{montasser2019vc} with the loss replaced. 
\end{proof}

\begin{theorem}[Sample Complexity for Agnostic Robust PAC Learning under Rejection]\label{thm:sample-complexity-ind-agn}
In the agnostic setting, for any $\calH = \calH_c \times \calH_r$ and $\mathcal{U}$, and any $ \epsilon, \delta \in (0, 1/2)$, 
\begin{align}
    \MAG(\epsilon, \delta; \calH, \mathcal{U}) 
    &= \calO\biggl(\VC(T(\calH)) \VC^*(T(\calH)) \log \left(\VC(T(\calH)) \VC^*(T(\calH))\right)\\
    & \qquad \frac{1}{\varepsilon^2} \log ^2\left(\frac{\VC(T(\calH)) \VC^*(T(\calH))}{\varepsilon}\right)+\frac{1}{\varepsilon^2} \log \left(\frac{1}{\delta}\right)\biggr) \\
    &= 2^{O(\VC(\mathcal{H}))} \frac{1}{\varepsilon^2} \log ^2\left(\frac{1}{\varepsilon}\right)+O\left(\frac{1}{\varepsilon^2} \log \left(\frac{1}{\delta}\right)\right) \label{eqn:ind-agn-complexity}\\
    &= 2^{O\left(\left(d_r+d_c\right)\log \left(d_r + d_c\right) \right)} \frac{1}{\varepsilon^2} \log ^2\left(\frac{1}{\varepsilon}\right)+O\left(\frac{1}{\varepsilon^2} \log \left(\frac{1}{\delta}\right)\right)
\end{align}
where $d_r = \VC(\calH_r), d_c = \VC(\calH_c)$.
\end{theorem}
\begin{proof}
Combining results from Lemma~\Cref{lem:inductive-real-agnostic} and \Cref{thm:sample-complexity-ind-real} gives the complexity result. 

Solving \Cref{eqn:ind-agn-complexity} gives the following generalization result given in \Cref{tab:bounds}
\[
\underset{\substack{(\bx, \by) \sim \mathcal{D}^{n}}}{\Pr}\left[ \mathrm{R}_{\calU}^{\rej} (\calA(\bx,\by); \calD) \le \epsilon  \right] \geq 1-\delta
\]

where $\epsilon = \calO\left(\sqrt{\frac{2^{\VC\left(T(\calH)\right)} +\log(1/\delta)}{n}}\right)$. 
\end{proof}

\subsection{Transduction+Rejection: Realizable Case}\label{app:trans-real}

We will prove a more general result which then implies Theorem~\ref{thm:rejection-simplified-realizable}. First,  the training data can also be perturbed, i.e., the adversary perturbs $\bz \in \U(\bx)$ and $\tilde{\bz} \in \U(\tilde{\bx})$, and the learner $\mathbb{A}$ are given $(\bz,\by, \tilde{\bz})$ instead of $(\bx,\by, \tilde{\bz})$. The criterion in the transductive rejection error (see \Cref{tab:errs}) is then the worst case over both $\bz \in \U(\bx)$ and $\tilde{\bz} \in \U(\tilde{\bx})$. 
Second, we will consider $\OPT_{\U^3} = 0$ and prove the guarantee tolerating $\U^2$. This then implies the guarantee tolerating $\U$ when $\OPT_{\U^{3/2}} = 0$.

In general the set of optimally learned classifiers $\Delta$ is defined as follows  \cite{montasser2021transductive}:
\[
\Delta_{\calH}^{\mathcal{U}}(\bz, \by, \tilde{\bz})= 
\begin{cases} 
\left\{h \in \calH: \Rt_{\mathcal{U}^{-1}}(h ; \bz, \by)=0 \wedge \Rt_{\mathcal{U}^{-1}}(h ; \tilde{\bz})=0\right\} & \text{(Realizable Case)}\\
\underset{h \in \calH}{\argmin} \max \left\{ \Rt_{\mathcal{U}^{-1}}(h ; \bz, \by) , \Rt_{\mathcal{U}^{-1}}(h ; \tilde{\bz})\right\}  & \text{(Agnostic Case)}\\
\end{cases}
\]
where \[\operatorname{R}_{\U}(h;\bz, \by) = \sup_{\tbx \in \U(\bz)}{1\over n}\sum_{i=1}^n \mathbb{1}\{h(\tx_i) \neq y_i\}\] and \[\operatorname{R}_{\U}(h;\bz) = \operatorname{R}_{\U}(h;\bz, h(\bz)).\]

Recall the transformation $F$ which we define following Tram\`er \cite{tramer2022detecting} in \Cref{sec:theory}.

Then, we define the \textit{relaxed robust shattering dimension} following \cite{montasser2021transductive}: 
\begin{definition}[Relaxed Robust Shattering Dimension] A sequence $z_1, \dots, z_k \in \calX$ is \textit{relaxed $\calU$-robustly shattered} by $\calH$, if $\forall y_1, \dots, y_k \in \{\pm 1\}$: $\exists x_1^{y_1}, \dots, x_k^{y_k} \in \calX$ and $\exists h \in \calH$ such that $z_i\in \calU(x_i^{y_i})$ and $h(\calU(x_i^{y_i})) = y_i, \; \forall 1\leq i\leq k$. The \textit{relaxed $\calU$-robust shattering dimension} $\rdim_{\calU}(\calH)$ is defined as the largest $k$ for which there exist $k$ points that are relaxed $\calU$-robustly shattered by $\calH$. 
\end{definition}

Define the set of \textit{intermediate perturbations} as follows:
\begin{definition}[Intermediate Perturbations] Given $x$ and $z$ and perturbations $\U_1$ and $\U_2$, the set of possible intermediate perturbations between $x$ and $z$ is
\[\ip_{\U_1, \U_2}(x, z) = \begin{cases}
    \{x\} & \text{if } x = z\\
    \U_1(x) \cap \U_2^{-1}(z) & \text{otherwise}
\end{cases}\]
\end{definition}

\begin{theorem}\label{thm:transductive-realizable}
For any $n \in \mathbb{N}$, $\delta > 0$, class $\calH$, perturbation set $\CalU$,
and distribution $\CalD$ over $\CalX \times \CalY$ satisfying ${\OPT}_{\U^{-1}\U} = 0$: 

\[
\underset{\substack{(\bx, \by) \sim \mathcal{D}^{n} \\(\tbx, \tby) \sim \mathcal{D}^{n}}}{\Pr}\left[
\begin{array}{l}
\forall \bz \in \UC(\bx), \forall \bz_0 \in \ipu(\bx, \bz),
\forall \tbz \in \UC(\tbx), \forall \tbz_0 \in \ipu(\tbx, \tbz),\\
\forall \hat{h} \in \FU\left(\Delta_{\calH}^{\mathcal{U}}(\bz_0, \by, \tbz_0)\right): 
\erej(\hh; \bx, \by, \tilde{\bx}, \tilde{\bz}, \tilde{\by}) \leq \epsilon
\end{array}
\right] \geq 1-\delta
\]

where $\epsilon = \frac{\rdim_{\calU^{-1}}(\calH) \log (2n) + \log(1/\delta)}{n} \leq \frac{\VC(\calH)\log(2n)+ \log (1/\delta) }{n} $.
\end{theorem}

\begin{proof}

We adapt the strategy of Theorem 5 of~\cite{tramer2022detecting} for the rejection scenario.

By setting $\bz = \bz_0, \tbz = \tbz_0$ and applying Theorem 1 of~\cite{montasser2021transductive}, we obtain the following
\begin{equation}\label{eqn:montasser-realizable}
\underset{\substack{(\bx, \by) \sim \mathcal{D}^{n} \\(\tbx, \tby) \sim \mathcal{D}^{n}}}{\Pr}\left[\forall \bz_0 \in \mathcal{U}(\bx), \forall \tbz_0 \in \mathcal{U}(\tbx),  
\forall h \in \Delta_{\calH}^{\mathcal{U}}(\bz_0, \by, \tbz_0): \operatorname{err}_{\tilde{\bz}_0, \tilde{\by}}(h) \leq \epsilon\right] \geq 1-\delta
\end{equation}
as ${\OPT}_{\CalU^{-1}(\CalU)} = 0$.

Suppose $(\bx,\by), (\tbx, \tby) \sim \mathcal{D}^n$. Now, let $\bz \in \UC(\bx), \tbz \in \UC(\tbx)$ and take some $\bz_0 \in \ipu(\bx, \bz), \tbz_0 \in \ipu(\tbx, \tbz)$, both of which are necessarily nonempty as $\UC = \US\U$, and $\hh \in \FU\left(\Delta_{\calH}^{\mathcal{U}}(\bz_0, \by, \tbz_0) \right)$.

Write $\hh = \FU(h)$ for some $h \in \Delta_{\calH}^{\mathcal{U}}(\bz_0, \by, \tbz_0)$.

From \Cref{eqn:montasser-realizable} (replacing $\bz$ with $\bz_0$ and $\tbz$ with $\tbz_0$), it is enough to show that 
\[\erej(\hh; \bx, \by, \tilde{\bx}, \tilde{\bz}, \tilde{\by}) \leq \operatorname{err}_{\tbz_0, \tby}(h).\]

Suppose that $\hh$ incurs an error under rejection at point $\tz_i$; it is enough to show that $h$ incurs an error at $\tzi$. Furthermore, note that because $h \in \Delta_{\calH}^{\mathcal{U}}(\bz_0, \by, \tbz_0)$, we have that $h(\U^{-1}(\tzi)) = \{h(\tzi)\}$ as $\tzi \in \U^{-1}(\tzi)$. Write $h(\tzi) = \hyi$. 

We have one of the following:
\begin{enumerate}
    \item $\hh(\tz_i) \neq \ty_i$ and $\tz_i = \tx_i$
    \item $\hh(\tz_i) \notin \{\ty_i, \perp\}$ and $\tz_i \neq \tx_i$
\end{enumerate}

In the first case, we must have $\tzi = \tx_i$ as well as $\tzi$ is an intermediate perturbation between $\tx_i$ and $\tz_i$, so, as $h(\U^{-1}(\tz_i)) = h(\U^{-1}(\tzi))  = \hyi$, $\hh$ does not reject $\tzi$ and $\hh(\tzi) = \hyi$. Hence, $h(\tzi) = \hyi$ as well so, as $\hh$ makes an error at $\tz_i$, $\hyi \neq y$ and so $h$ makes an error at $\tzi$.

In the second case, if $h(\U^{-1}(\tz_i)) \neq \{h(\tz_i)\}$, then $\hh$ would reject $\tz_i$ and hence not incur an error. So $h(\U^{-1}(\tz_i)) = \{h(\tz_i)\}$ and so $\hh(\tz_i)  = h(\tz_i)$. Since $\tzi \in \U(\tx_i) \cap \USI(\tz_i)$, there exists some $\tzi' \elem \U(\tzi) \cap \UI(\tz_i)$ and so, $h(\tzi) = h(\tzi') = h(\tz_i) = \hh(\tz_i) = \hyi$, so $h$ incurs an error at $\tzi$.

In either case, we have that $h$ makes an error at $\tzi$, showing the result.
\end{proof}

\paragraph{Remark:} More direct approaches may seem possible, but have surprising pitfalls. At first glance, this approach may seem less natural than simply applying the analysis of \citep{montasser2021transductive} to a potential $\tbz' \in \UH(\tbx)$ with the condition of $\OPT_\U$, obtaining a $\UH$-robust classifer $h'$, and deriving an $\epsilon$-robust selective classifier by the transformation $\FUH$.
While this seems possible at first, as \citep{tramer2022detecting} shows that applying this transformation results in doubled robustness, this isn't possible in this situation, as $h'$ is only guaranteed to be $\UH$-robust at $\tbz'$, not at \textit{every} $\epsilon / 2$ perturbation of $\tbx$ as needed by the analysis. Similarly, it might seem possible to obtain an $\epsilon / 2$-robust classifier at $\tbz$ using \citep{montasser2021transductive}, and derive the desired $\epsilon$-robust classifier from $\FUH$; this, however, requires the condition $\OPT_{\U^2}$, as the analysis of \citep{montasser2021transductive} only applies on perturbations up to half the margin; hence, this approach gains no advantage from rejection. 

\paragraph{Sample Complexity}
Given $\epsilon$ and $\delta$, we need 
\[\frac{\rdim_{\calU^{-1}}(\calH) \log (2n) + \log(1/\delta)}{n} \leq \epsilon\] for the result to hold.

Now, noting that $\log(2n) = 1 + \log n \leq 1 + \sqrt{n}$ for $n \geq 16$; hence we need to solve for the $n$ such that
\[\frac{\rdim_{\calU^{-1}}(\calH) (1 + \sqrt{n}) + \log(1/\delta)}{n} = \epsilon\]
or, equivalently
\[{\CC + \DD + \sqrt{n}\over n} = \epsilon\]
or
\[\sqrt{n} = n\epsilon - \CC - \DD\]
or
\[n = n^2\epsilon^2 - 2\epsilon\left(\CC + \DD\right)n + \left(\CC + \DD\right)^2\]
or
\[n^2\epsilon^2 - \left(2\epsilon\left(\CC + \DD\right) + 1\right) n + \left(\CC + \DD\right)^2 = 0.\]
Solving, the result holds if
\begin{align*}
    n &\geq {2\epsilon\left(\CC + \DD\right) + 1 + \sqrt{(2\epsilon\left(\CC + \DD\right) + 1)^2 - 4 \left(\CC + \DD\right)^2\epsilon^2}\over 2\epsilon^2}\\
    &= O\left({\CC + \DD \over \epsilon} + {\sqrt{\CC + \DD} \over \epsilon^{3\over 2}}\right)
\end{align*}

and, similarly, using 
\[\frac{\rdim_{\calU^{-1}}(\calH) \log (2n) + \log(1/\delta)}{n} \leq \frac{\VC(\calH)\log(2n)+ \log (1/\delta) }{n}\]
we have the result if 
\[
    n = O\left({\VC(\calH) + \DD \over \epsilon} + {\sqrt{\VC(\calH) + \DD} \over \epsilon^{3\over 2}}\right)
\]

\paragraph{Remark: } If ${\OPT}_{\U^{-1}\U} = 0$, we can guarantee the existence of an $\hh$ which satisfies our conditions, but we can't guarantee that we will find it, as we cannot find $\Delta_{\calH}^{\mathcal{U}}(\bz_0, \by, \tbz_0)$ without $\bz_0$ and $\tbz_0$. We can, however, construct that an algorithm which, if it returns a model, always returns on which meets the conditions.

\paragraph{Simplified Result}
To obtain a bound which does not involve an intermediate perturbation step, we may let
\[ \Delta_{\rej, \calH}^{\U}(\bz, \by, \tilde{\bz}) := \begin{cases}
\hat{\Delta}  \cup \Delta_{\rej, \calH}^{\U\prime}(\bz, \by, \tilde{\bz}) & \hat{\Delta}_{\calH}^{\mathcal{U}}(\bz, \by, \tilde{\bz})(\tbz)| = 1\text{, and}\\
  l\Delta_{\rej, \calH}^{\U\prime}(\bz, \by, \tilde{\bz}) & \text{ otherwise }\end{cases}
\]
where
\[\Delta_{\rej, \calH}^{\mathcal{U}\prime}(\bz, \by, \tilde{\bz}) = \bigcap_{\tbz' \in \USI(\tbz)}\Delta_{\calH}^{\U}(\bz, \by, \tbz')\]
where
\[\hat{\Delta}_{\calH}^{\mathcal{U}}(\bz, \by, \tilde{\bz}) = \bigcup_{\tbz' \in \USI(\tbz)}\Delta_{\calH}^{\U}(\bz, \by, \tbz').\]
If $|\hat{\Delta}_{\calH}^{\mathcal{U}}(\bz, \by, \tilde{\bz})(\tbz)| = 1$, then as $\Delta_{\calH}^{\mathcal{U}}(\bz_0, \by, \tbz_0)(\tbz) \subseteq \hat{\Delta}_{\calH}^{\mathcal{U}}(\bz, \by, \tilde{\bz})(\tbz)$, $\hat{\Delta}_{\calH}^{\mathcal{U}}(\bz, \by, \tilde{\bz})(\tbz) = \Delta_{\calH}^{\mathcal{U}}(\bz_0, \by, \tbz_0)(\tbz)$ since $\Delta_{\calH}^{\mathcal{U}}(\bz_0, \by, \tbz_0)$ is nonempty as ${\OPT}_{\CalU^{-1}(\CalU)} = 0$.

Note that for common classes of perturbations, we can simplify the $\Delta'_{\rej}$. Note that the conditions of the theorem hold for perturbations defined via $\epsilon$-balls in a metric.

Let
\[\Delta_{\calH}^{\mathcal{U},\mathcal{U'}}(\bz, \by, \tilde{\bz})= 
\left\{h \in \calH: \Rt_{\mathcal{U}^{-1}}(h ; \bz, \by)=0 \wedge \Rt_{\mathcal{U}'^{-1}}(h ; \tilde{\bz})=0\right\}.\]

\begin{lemma}\label{lem:rejection-delta-realizable} In the realizable case, if $\U$ = $\U^{-1}$,
\[ \Delta_{\rej, \calH}^{\mathcal{U}\prime}(\bz, \by, \tilde{\bz}) = \Delta_{\calH}^{\U, \U^3}(\bz, \by, \tilde{\bz})\]
\end{lemma}

\begin{proof}

Suppose $h \in \Delta_{\rej, \calH}^{\mathcal{U}\prime}(\bz, \by, \tilde{\bz})$. Then by the definitions of $\Delta_{\rej}$ and $\Delta$, $\Rt_{\mathcal{U}^{-1}}(h; \bz, \by) = 0$ and for any $\bz' \in \USI(\bz),\, \tbz' \in \USI(\tbz)$, we have that, for any $\bx \in \U^{-1}(\bz')$ and $\tbx \in \U^{-1}(\tbz')$, $h(x_i) = h(z'_i)$ and $h(\tx_i) = h(\tz'_i)$. Now, as there exists some $\bz'' \in \U(\bz') \cap \UI(bz)$ and $h(\bx) = h(\bz') = h(\bz'') = h(\bz)$ by an argument similar to that in~\Cref{thm:transductive-realizable} and similarly for $\tbx$ and $\tbz$, we have that
for any $\bx \in \UCI(\bz)$ and $\tbx \in \UCI(\tbz)$, $h(x_i) = h(z_i)$ and $h(\tx_i) = h(\tz_i)$, and so
\[\Delta_{\rej, \calH}^{\mathcal{U}\prime}(\bz, \by, \tilde{\bz}) \subseteq \Delta_{\calH}^{\U, \UC}(\bz, \by, \tilde{\bz})\]

Now, if $h \in \Delta_{\calH}^{\U, \UC}(\bz, \by, \tilde{\bz})$, we have that, $\Rt_{\mathcal{U}^{-1}}(h; \bz, \by) = 0$ and for any $\bx \in \UCI(\bz)$ and $\tbx \in \UCI(\tbz)$, $h(x_i) = h(z_i)$ and $h(\tx_i) = h(\tz_i)$. Now, suppose $\bz' \in \USI(\bz),\, \tbz' \in \USI(\tbz)$. Since $x \in \U(x)$ for all $x$, $\bz' \in \UCI(\bz),\, \tbz' \in \UCI(\tbz)$ as well. 
Hence, $h(z'_i) = h(z_i)$ and $h(\tz'_i) = h(\tz_i)$. Now, if $\bx \in \U^{-1}(\bz')$ and $\tbx \in \U^{-1}(\tbz')$, we have $\bx \in \UCI(\bz)$ and $\tbx \in \UCI(\tbz)$ and so $h(x_i) = h(z_i)$ and $h(\tx_i) = h(\tz_i)$. But then $h(x_i) = h(z'_i)$ and $h(\tx_i) = h(\tz'_i)$. Hence, we have that 
\[ \Delta_{\calH}^{\U,\UC}(\bz, \by, \tilde{\bz}) \subseteq \Delta_{\rej, \calH}^{\mathcal{U}\prime}(\bz, \by, \tilde{\bz})\]
and the result follows.
\end{proof}

From this, we immediately derive the corollary
\[ \Delta_{\rej, \calH}^{\mathcal{U}}(\bz, \by, \tilde{\bz}) \supseteq \Delta_{\calH}^{\U^3}(\bz, \by, \tilde{\bz}).\]

\paragraph{Remark:} Note that this means that $\Delta_{\rej, \calH}^{\mathcal{U}\prime}(\bz, \by, \tilde{\bz})$ is nonempty if $\OPT_{\U^6} = 0$, and, by the definition of $\Delta$, $\Delta$ is also nonempty if $|\hat{\Delta}_{\calH}^{\mathcal{U}}(\bz, \by, \tilde{\bz})(\tbz)| = 1$, i.e. if there exists only one possible labeling of the $\tbz$ which is robust at some possible intermediate perturbation.

Now, by the above and from \Cref{thm:transductive-realizable} we may immediately derive \Cref{thm:rejection-simplified-realizable} by noting that if $\U = \U^{-1}$, $\U^{-1}\U = \U^2$, and
if $\hh \in \mathrm{F}_{\calU
}(\Delta_{\calH}^{\U}(\bz, \by, \tilde{\bz})) = \FUT(\Delta_{\rej, \calH}^{\UT}(\bz, \by, \tilde{\bz}))$ then we have $\hh \in \FUT\left(\Delta_{\calH}^{\UT}(\bz_0, \by, \tbz_0)\right)$ for some $\bz_0 \in \iput(\bx,\bz)$ and $\tbz_0 \in \iput(\tbx, \tbz)$.

Furthermore, following from~\Cref{lem:rejection-delta-realizable}, $\Delta_{\rej, \calH}^{\UT}(\bz, \by, \tilde{\bz})$ is nonempty is $OPT_\US = 0$, showing completeness that the $\Delta$ of~\Cref{thm:rejection-simplified-realizable} is nonempty under that condition, as well as, as noted above, under the condition that there exists only one possible labeling consistent on a potential intermediate perturbation.

Now, we demonstrate that there exists a data distribution for which the transductive learner implied by $\Delta$ finds a solution for which the bound applies, but where no transductive learner has zero asymptotic robust error
\begin{theorem}
  There exists a distribution $\CalD$ over $\CalX \times \CalY$, a hypothesis class $\calH$, and perturbation set $\U$ for which, with probability $\ge 1 - 2^{1-n}$, for any $(\bx, \by), (\tbx, \tby) \sim \mathcal{D}^n$ and any $\tbz \in \UC(\tbz)$, \ $\Delta_{\rej, \calH}^{\mathcal{U}}(\bx, \by, \tilde{\bz})$ is nonempty and for all $h \in \Delta_{\rej, \calH}^{\mathcal{U}}(\bz, \by, \tilde{\bz})$, $\operatorname{err}^{\rej}_{\U}(h; \bx, \by, \tilde{\bx}, \tbz, \tilde{\by}) = 0$
  but, there exists no transductive learner (without rejection) $\mathbb{A}$ for which 
  $\lim_{n \to \infty}\mathbb{E}\left[\sup_{\tbz \in \U(\tbx)}\operatorname{err}_{\U}(\mathbb{A}(\bx,\by,\tbz); \bx, \by, \tilde{\bz}, \tilde{\by})\right] < 1/2$.
\end{theorem}
\begin{proof}

  Consider the simple discrete distribution $\D$ with $(x, y) \sim \D$ is $(1, 1)$ with probability 1/2 and $(-1, 0)$  with probability 1/2. Now, let $\U(x) = \{y \mid |y - x| < 1.5\}$ and let $\calH$ be the class of 

  Now, let $\calH$ be the class of threshold functions $h_w(x) = \mathbb{1}_{x \ge w}$ and $h_w^-(x) = \mathbb{1}_{x < w}$ for integer $w$.

  First, note that with probability $1 - 2^{1-n}$ both $(-1, 0)$ and $(1,1)$ appear in $\bx$. In that case, any $h \in \Delta_{\calH}^{\U}(\bx, \by, \tbz')$ must be robust at $-1$ and $1$ up to a radius of 1/2; and hence $h$ must be $h_w$ for some $w \in [-1/2, 1/2]$ (and hence, $w = 0$). Hence, $|\hat{\Delta}| \le 1$; note that for any possible perturbation of $-1$ or $1$ is within $\U^2$ (i.e. within 1 unit of) either $-1$ or $1$; hence, there always exists some $\tbz'$ where $\Delta_{\calH}^{\U}(\bx, \by, \tbz')$ is nonempty.

  But then, there must exist exactly one element in $\hat{\Delta}$, and so $\Delta$ is nonempty. Consider $\tbz_i$. We have two cases:
  
  If $\tbz_i \in [-1,-1/2] \cup [1/2, 1]$, then, as $h$ is robustly correct with radius 1/2 about $1$ and $-1$, then $\tbx_i = \sign(\tbz_i)$ and hence $h(\tbx_i) = \sign(\tbz_i)$. If $\tbx_i = \tbz_i$ we do not reject as $h$ is robust with radius 1/2 about $-1$ and $1$. Thus, we do not incur an error at $\tbz_i$.

  If $\tbz_i \in (-1/2, 1/2)$, then $\tbz_i$ must be perturbed. But we have both positive and negative values within 1/2 of $\tbz_i$, and so $F_\U(\tbz_i) = \perp$. Hence, we do not occur an error at $\tbz_i$.

  In all cases, we do not incur an error if both $x=-1$ and $x=1$ appear in the training data, and so $\operatorname{err}^{\rej}_{\U}(h; \bx, \by, \tilde{\bx}, \tbz, \tilde{\by})$ is 0 with probability $\ge 1 - 2^{1- n}$.

  To see that there exists no transductive algorithm (without rejection) that can have asympotic error below 1/2, note that any $\tbx$ can be perturbed to $\tbz$ where all $\tbz$ are 0; hence, samples from class $0$ and class $1$ are indistinguishable and the minimum error on $\tbz$ achievable by $h$ is the minimum of the fraction of the $\tbx$ which are $-1$ and the fraction which are $1$. As $n \rightarrow \infty$, these both tend to 1/2 and the result follows.
\end{proof}

\subsection{Transduction+Rejection: Agnostic Case}\label{app:trans-agn}

Note that, if $\U$ can be decomposed into a form $\U = (\U^{1/3})^3$ where $\U^{1/3} = \U^{-1/3}$ (as with standard perturbations in $l_p$), we obtain a bound which depends on OPT$_{\U^{2/3}}$ rather than OPT$_{\U^2}$, enabling, for $\hh$ satisfying the conditions, much stronger guarantees if OPT$_{\U^{2/3}} << \OPT_{\U^2}$. Note that as $\forall x\, x \in \U(x)$, $\forall x\, \U^{2/3}(x) \subseteq \U^2(x)$, and so OPT$_{\U^{2/3}} \leq \OPT_{\U^2}$.

\begin{theorem}\label{thm:transduction-agnostic}
For any $n \in \mathbb{N}$, $\delta > 0$, class $\calH$, perturbation set $\CalU$, and distribution $\CalD$ over $\CalX \times \CalY$: 

\[
\underset{\substack{(\bx, \by) \sim \mathcal{D}^{n} \\(\tbx, \tby) \sim \mathcal{D}^{n}}}{\Pr}\left[
\begin{array}{l}
\forall \bz \in \UC(\bx), \forall \bz_0 \in \ipu(\bx, \bz),
\forall \tbz \in \UC(\tbx), \forall \tbz_0 \in \ipu(\tbx, \tbz),\\
\forall \hat{h} \in \FU\left(\Delta_{\calH}^{\mathcal{U}}(\bz_0, \by, \tbz_0)\right): 
\erej(\hh; \bx, \by, \tilde{\bx}, \tilde{\bz}, \tilde{\by}) \leq \epsilon
\end{array}
\right] \geq 1-\delta
\]

where 
\[\epsilon=\min \left\{2 \OPT_{\U^{-1}\U}+O\left(\sqrt{\frac{\VC(\calH)+\log (1 / \delta)}{n}}\right), 3 {\OPT}_{\U^{-1}\U}+O\left(\sqrt{\frac{\operatorname{rdim} \mathcal{U}(\calH) \ln (2 n)+\ln (1 / \delta)}{n}}\right)\right\}.\]
\end{theorem}

\begin{proof}
Suppose $(\bx,\by), (\tbx, \tby) \sim \mathcal{D}^n$. Now, let $\bz \in \UC(\bx), \tbz \in \UC(\tbx)$ and take some $\bz_0 \in \ipu(\bx, \bz), \tbz_0 \in \ipu(\tbx, \tbz)$, both of which are necessarily nonempty, and $\hh \in \FU\left(\Delta_{\calH}^{\mathcal{U}}(\bz_0, \by, \tbz_0) \right)$.

Write $\hh = \FU(h)$ for some $h \in \Delta_{\calH}^{\mathcal{U}}(\bz_0, \by, \tbz_0)$.

We will begin as in~\Cref{thm:transductive-realizable}. As before, there are two cases in which $\hh$ can incur an error at $\tz_i$:
\begin{enumerate}
    \item $\hh(\tz_i) \neq \ty_i$ and $\tz_i = \tx_i$
    \item $\hh(\tz_i) \notin \{\ty_i, \perp\}$ and $\tz_i \neq \tx_i$
\end{enumerate}

Now, if $\tz_i = \tx_i$, an error occurs if $\hh$ rejects $\tz_i$ or if $h$ robustly predicts some $\hyi \neq \ty_i$; hence an error occurs if $h$ is not $\UI$-robust at $\tzi$ or if $h(\tzi) \neq \ty_i$.

Otherwise, $h$ must be $\UI$-robust at $\tz_i$, as, otherwise, $\hh$ would reject $\tz_i$. Hence, as there exists some $\tzi' \in \U(\tzi) \cap \UI(\tz_i)$, if $h$ is $\U$-robust at $\tzi$, we must have $h(\tz_i) = h(\tzi)$, and so, if $\hh$ makes an error, $h$ is not $\UI$-robust at $\tzi$ or $h(\tzi) \neq \ty_i$.

Now, in both cases, errors only occur if $h$ is not $\UI$-robust at $\tzi$ or $h(\tzi) \neq \ty_i$. As $\tx_i \in \UI(\tzi)$, we have, equivalently, that an error occurs if $h$ is not $\UI$-robust at $\tzi$ or $h(\tx_i) \neq \ty_i$.

Hence,
\[\erej(\hh; \bx, \by, \tilde{\bx}, \tilde{\bz}, \tilde{\by}) \leq \erej(h; \tx, \ty) + \mathrm{R}_{\UI}(h;\tbz_0)\]

Now, the right hand is exactly what is bounded in Theorem 2 of~\cite{montasser2021transductive}; as we have $h \in \Delta_{\calH}^{\mathcal{U}}(\bz_0, \by, \tbz_0)$, we have 
\[\erej(\hh; \bx, \by, \tilde{\bx}, \tilde{\bz}, \tilde{\by}) \leq \erej(h; \tx, \ty) + \mathrm{R}_{\UI}(h;\tbz_0) \leq \epsilon\]
where
\[\epsilon=\min \left\{2 \OPT_{\U^{-1}\U}+O\left(\sqrt{\frac{\VC(\calH)+\log (1 / \delta)}{n}}\right), 3 {\OPT}_{\U^{-1}\U}+O\left(\sqrt{\frac{\operatorname{rdim} \mathcal{U}(\calH) \ln (2 n)+\ln (1 / \delta)}{n}}\right)\right\}\]
with probability $\geq 1 - \delta$ by its proof.
\end{proof}

As in the realizable case, we can immediately derive the following corollary. However, we cannot simplify the definition of $\Delta_\text{rej}$ as before; see Lemma~\ref{lem:rejection-agnostic-delta}.

\begin{corollary}\label{corr:transductive-agnostic}
For any $n \in \mathbb{N}$, $\delta > 0$, class $\calH$, perturbation set $\CalU$ where $\U = \U^{-1}$, and distribution $\CalD$ over $\CalX \times \CalY$:

\[\underset{\substack{(\bx, \by) \sim \mathcal{D}^{n} \\(\tbx, \tby) \sim \mathcal{D}^{n}}}{\Pr}\left[
\begin{array}{l}
\forall \bz \in \UC(\bx), \forall \tbz \in \UC(\tbx),
\forall \hat{h} \in \FU\left(\Delta_{\rej, \calH}^{\mathcal{U}}(\bz, \by, \tilde{\bz})\right):  \\
\erej(\hh; \bx, \by, \tilde{\bx}, \tilde{\bz}, \tilde{\by}) \leq \epsilon
\end{array}
\right] \geq 1-\delta
\]

where
\[\epsilon=\min \left\{2 \OPT_{\U^{-1}\U}+O\left(\sqrt{\frac{\VC(\calH)+\log (1 / \delta)}{n}}\right), 3 {\OPT}_{\U^{-1}\U}+O\left(\sqrt{\frac{\operatorname{rdim} \mathcal{U}(\calH) \ln (2 n)+\ln (1 / \delta)}{n}}\right)\right\}.\]
\end{corollary}

\begin{lemma}\label{lem:rejection-agnostic-delta} In the agnostic case, we have that if $\U = \U^{-1}$,
\[ \Delta_{\rej, \calH}^{\mathcal{U}}(\bz, \by, \tilde{\bz}) \subseteq \Delta_{\calH}^{\UC}(\bz, \by, \tilde{\bz})\]
\end{lemma}

\begin{proof}
By the definition of $\Rt$, we have
\begin{align*}
\Rt_{\UCI}(h ; \tbz) &= \frac{1}{n} \sum_{i=1}^{n} \mathbb{1}\left\{\exists \tilde{x}_{i} \in \UCI\left(\tz_{i}\right): h\left(\tilde{x}_{i}\right) \neq h\left(\tz_{i}\right)\right\} \\
&= \frac{1}{n} \sum_{i=1}^{n} \mathbb{1}\left\{\begin{array}{l}
\exists \tilde{z}_{i}' \in \USI\left(\tz_{i}\right)\exists \tilde{x}_{i} \in \mathcal{U}^{-1}\left(\tz_{i}'\right):
h\left(\tilde{x}_{i}\right) \neq h\left(\tz_{i}\right)
\end{array}\right\}\\
&= \underset{\tilde{z}_{i}' \in \USI\left(\tz_{i}\right)}{\max} \frac{1}{n} \sum_{i=1}^{n} \mathbb{1}\left\{\exists \tilde{x}_{i} \in \mathcal{U}^{-1}\left(\tz_{i}'\right): h\left(\tilde{x}_{i}\right) \neq h\left(\tz_{i}\right)\right\}\\
&= \underset{\tilde{z}_{i}' \in \USI\left(\tz_{i}\right)}{\max} \Rt_{\mathcal{U}^{-1}}(h ; \tbz')
\end{align*}
where the last equality holds as $x \in \U(x)$ for all $x$ and as $\U = \U^{-1}$, which together show that if for some $\tz_i$ and $\tz_i' \in \USI(\tz_i)$ we have that $h(\tz_i') \neq h(\tz_i)$, that either there exists some $\tz_i'' \in \U = \UI(\tz_i')$ such that $h(\tz_i'') \neq h(\tz_i')$ or there exists some $\tz_i'' \in \U = \UI(\tz_i)$ such that $h(\tz_i'') \neq h(\tz_i)$ (as before, note that $\tz_i = \U(\tz_i'')$ for some $\tz_i'' \elem \U(\tz_i')$ by the definition of $\UC$); the reverse is similar.

We can derive a result for $\Rt_{\UCI}(h;\bz,\by)$ similarly.

Suppose $h \in \Delta_{\rej, \calH}^{\mathcal{U}}(\bz, \by, \tilde{\bz})$. Then, $h$ minimizes $\max \left\{ \Rt_{\mathcal{U}^{-1}}(h ; \bz', \by) , \Rt_{\mathcal{U}^{-1}}(h ; \tbz')\right\}$ for all $\bz' \in \USI(\bz),\, \tbz' \in \USI(\tbz)$, so by the above, $h$ must also minimize
\begin{align*}
&\underset{\bz' \in \USI(\bz), \tbz' \in \USI(\tbz)}{\max}\max\left\{\Rt_{\mathcal{U}^{-1}}(h ; \bz', \by) , \Rt_{\mathcal{U}^{-1}}(h ; \tbz')\right\} \\
&= \max\left\{\underset{\bz' \in \USI(\bz)}{\max} \Rt_{\mathcal{U}^{-1}}(h ; \bz', \by) , \underset{\tbz' \in \USI(\tbz)}{\max}\Rt_{\mathcal{U}^{-1}}(h ; \tbz')\right\}\\
&= \max\left\{\Rt_{\UCI}(h;\tbz), \mathrm{R}_{\UCI}(h;\bz, \by) \right\}
\end{align*}
and so $h \in \Delta_{\calH}^{\UC}(\bz, \by, \tilde{\bz})$.

However, minimizing
\[\underset{\bz' \in \USI(\bz), \tbz' \in \USI(\tbz)}{\max}\max\left\{\Rt_{\mathcal{U}^{-1}}(h ; \bz', \by) , \Rt_{\mathcal{U}^{-1}}(h ; \tbz')\right\}\]
does not necessarily imply that $h$ minimizes $\max\left\{\Rt_{\mathcal{U}^{-1}}(h ; \bz', \by) , \Rt_{\mathcal{U}^{-1}}(h ; \tbz')\right\}$ for all $\bz' \in \USI(\bz), \tbz' \in \USI(\tbz)$, so the reverse may not hold.
\end{proof}

\subsection{Extension to Unbalanced Training and Test Data}\label{app:unbalanced-train-test}

We provide a sketch of a proof that allows extending Theorem 1 of \citep{montasser2021transductive} to unbalanced training and test sets; however, for simplicity, we will work with the original form. The assumptions are the same, except that we have $n$ training points and $m$ test points.

The proof is exactly as before up to the "Finite robust labelings" portion (which points are and are not labelled don't matter up to then and the symmetry arguments still apply). The basic idea of determining the probability of zero loss on the training and test sets and error $> \epsilon$ on the test examples with permutation still applies. Let $E_{\sigma,\bx}$ be the event that there exists a labelling $\hh(\bx_{\sigma(1:n+m)})$ in the allowable set where this occurs.

We have 
\[\Pr_\sigma\left[E_{\sigma, \bx}\right] \leq \Pr_\sigma\left[\exists \hh \in \PUH(x_1,\dots,x_{n+m}) : 
\err_{\bx_{\sigma(1:n)}, \by_{\sigma(1:n)}}(\hh) = 0 
\land \err_{\bx_{\sigma(n:n+m)}, \by_{\sigma(n:n+m)}}(\hh) > \epsilon \right]\]
and, as in \citep{montasser2021transductive}, note the probability of choosing such a perturbation $\sigma$ for a fixed $\hh$ is at most
\[\pmnm^s \leq \pmnm^{\ceil{\epsilon m}} = \pnmm^{-\ceil{\epsilon m}} \leq \pnmm^{\ceil{-\epsilon m}}\]
if we assume the number of total errors $s \geq \ceil{\epsilon m}$ without loss of generality (otherwise, $\err > \epsilon$ would be impossible).

Hence, by a union bound,
\[\Pr_\sigma\left[E_{\sigma, \bx}\right]\leq \left|\PUH(x_1,\dots,x_{n+m})\right|\pnmm^{\ceil{-\epsilon m}}\]
and so
\[\Pr_\sigma\left[E_{\sigma, \bx}\right]\leq
(n+m)^{\rdim_{\U^{-1}}(\HH)}
\pnmm^{\ceil{-\epsilon m}}\]
by Sauer's Lemma (in the form of  Lemma 3 of \citep{montasser2021transductive}).

Now, we bound the probability by $\delta$, we need 
\[(n+m)^{\rdim_{\U^{-1}}(\HH)}
\pnmm^{\ceil{-\epsilon m}}\leq \delta\]
which, solving, gives us
\[\epsilon \geq {\rdim_{\U^{-1}}(\HH) \log_{n+m\over m}(n+m) + \log_{n+m\over m} {1\over \delta}\over m} =
{\rdim_{\U^{-1}}(\HH) \log(n+m) + \log {1\over \delta}\over m\log\left(1+{m\over n}\right)}\]

Which reduces to the original result if $n=m$ (note that the logarithms are base-2).

\paragraph{Corollary}

If we fix $n+m$, $\HH$, and $\delta$, the guarantee is strongest (i.e. we minimize $\epsilon$) when $n=m$. To see this, consider the denominator. Write $\alpha = {m \over n}$. Then, we wish to maximize $n\alpha \log(1+\alpha)$ (or equivalently $f(\alpha) = \alpha\log(1+\alpha)$ subject to $\alpha \geq 0$. Now, note that $f'(\alpha) = \log(1+\alpha)-1 = 0$ when $\alpha = 1$, i.e. when $m=n$.

Also, we can see from the result above, that if we fix $m$ and $\delta$, then the minimum value of $\epsilon$ tends towards $\infty$ as $n\rightarrow\infty$, so there does not necessarily exist a labelled training set sampled from $\D$ which provides a guarantee with high probability of arbitrarily low error on a fixed test set.

\section{Experimental Details} \label{app:exp-detail}

\subsection{Computing Infrastructure} \label{app:compute-infra}
We used a SLURM cluster with A100 GPUs to run our experiments. 

\subsection{Baseline Details}

The baselines are trained with standard adversarial training \citep{goodfellow2015explaining} \citep{madry2018towards}. Attacks against AT without rejection use standard PGD with a cross-entropy objective, while attacks against AT with rejection use PGD targeting $\LREJ$ as described in algorithm~\ref{alg:rej-attack}. In all cases, the parameters for PGD in training are the same as those used in TLDR's training process for the same dataset.

\subsection{Defense}

In our implementation, we begin to incorporate the transductive term in our objective (see \Cref{eqn:transductive-loss}) after initially training the model with the inductive loss term only; this allows learning a better baseline before we begin to enforce robustness about the test points. In our experiments, we use the transductive loss in the final half of the training epochs.

\subsection{Adaptive Attack}
Solving for the perturbation $\tx$ by iteratively optimizing $\LREJ$ poses several difficulties.

First, the rejection-avoidance term $\left\|\tx -  \arg\max_{||x' - \tx|| \leq \epsilon} \LDB(x')\right\|$ is not differentiable with respect to $\tx$. While it is possible to approximate the derivative with the derivative of a proxy (e.g. differentiating though some fixed number of PGD steps, necessitating second-order optimization), this is extremely expensive and does not improve results in our experiments (see below).

Intuitively, we might see that this would be the case: if the decision boundary is smooth, we might expect the maximizers in $\U(x+\Delta)$ and $\U(x)$ to be the same for small $\Delta$ unless $x'$ is near the border of $\U(x)$ given that $\U(x+\Delta) \approx \U(x)$. In this case, approximating $x'$ as constant with respect to $x$ is reasonable.

In addition, note that if $h(x) = y$, the adversary must find a $\tx$ where $h(\tx) \neq y$ which is not rejected: if maximizing $\LREJ$ with PGD, the rejection-avoidance term penalizes moving $\tx$ towards the decision boundary. As this is necessary to find a valid attack (when $h(\tx) = y$ at initialization), we adjust $\lambda$ adaptively during optimization by setting it to zero when $h(\tx) = y$. 

\subsection{Transductive Attack Details}\label{app:attack-details}

We present two rejection-aware transductive attacks: a stronger but more computationally intensive rejection-aware GMSA (Algorithm~\ref{alg:rej-attack-gmsa}) and a weaker but faster rejection-aware transfer attack which takes the transductive robust rejection risk into account (Algorithm~\ref{alg:rej-attack-transfer}).

\begin{algorithm*}[htb]
  \small
  \caption{\textsc{\small Rejection-Aware GMSA}}
  \begin{algorithmic}[1]
    \REQUIRE A clean training set $T$, a clean test set $E$, a transductive learning algorithm for classifiers $\A$, an adversarial budget of $\epsilon$, $mode$ either MIN or AVG, a radius used for rejection $\epsilon_{\text{defense}}$, and a maximum number of iterations $N \geq 1$. $E|_X$ refers to the projection on the feature space for $E$.
    \STATE Search for a perturbation of the test set which fools the model space induced by $(T, \U(E|_X))$.\\
    \STATE $E' = E$
    \STATE $\hat{E} = E$
    \STATE $\err_\text{max} = -\inf$
    \FOR{i=0,\dots,N-1}
        \STATE Train a transductive model on the perturbed data.\\
        \STATE $h^{(i)} = \A(T, E'|_X)$
        \STATE \[\err = \frac{1}{|E'|} \sum_{i=1}^{|E'|} \mathbb{1}\left\{ 
            \left( F(h^{(i)})\left(\tilde{x}_{i}\right) \notin \{\tilde{y_i} \} \land \tilde{x}_{i} = x_{i}\right) 
            \lor \left( 
                F(h^{(i)})\left(\tilde{x}_{i}\right) \notin \{\tilde{y_i}, \perp\} \land \tilde{x}_{i} \neq x_{i}  \right) \right\}\]
        \COMMENT{The $\tilde{x}_i$ and the $x_i$ are the i$^{\text{th}}$ datapoints of $E'$ and $E$, repectively; $y_i$ is the true label.}
        
        \IF{$\err_\text{max} < \err$}
            \STATE $\hat{E} = E'$
        \ENDIF
        
        \FOR{$j=1,\dots,|E|$}
            \IF{$mode = \text{MIN}$}
                \STATE \[\tx_j = \arg\max_{\|\tx - x_j\| \leq \epsilon}\min_{1\leq k \leq i} \LREJ_{h^{(k)}}(\tx, y_j)\]
            \ELSE
                \STATE \[\tx_j = \arg\max_{\|\tx - x_j\| \leq \epsilon}{1\over i}\sum_{k=1}^i \LREJ_{h^{(k)}}(\tx, y_j)\]
            \ENDIF
            
            \COMMENT{Select whether to perturb by comparing success rates against past models for the clean and perturbed samples.}
            
            \STATE \[\err_{\text{clean}} = {1\over i}\sum_{0\leq k \leq i} \I\left[F\left(h^{(k)}\right)(x_j) \neq y_j\right]\]
            \STATE \[\err_{\text{perturbed}} = {1\over i}\sum_{0\leq k \leq i} \I\left[F\left(h^{(k)}\right)(\tx_j) \not\in \{y_j, \perp\}\right]\]
            
            \COMMENT{Do not perturb if the perturbation reduces robust rejection accuracy less on average than leaving the points unchanged.}
            \IF{$\err_{\text{perturbed}} < \err_{\text{clean}}$}
                \STATE $\tx_j = x_j$
            \ENDIF
            
            \STATE $E_j' = \tx_j, y_i$
        \ENDFOR
    \ENDFOR
    \STATE \textbf{Return:} $\hat{E}$
  \end{algorithmic}
  \label{alg:rej-attack-gmsa}
\end{algorithm*}

\begin{algorithm*}[htb]
  \small
  \caption{\textsc{\small Transductive Rejection-Aware Transfer}}
  \begin{algorithmic}[1]
    \REQUIRE A model $h$, a clean labelled test point $(x,y)$, an adversarial budget of $\epsilon$, and a radius used for rejection $\epsilon_{\text{defense}}$.\\
    
    \COMMENT{Search for a perturbation $\tx$ of $x$ for which $h$ predicts $\hat{y} \neq y$ robustly.}
    \STATE \[\tx = \arg\max_{\|\tx - x\| \leq \epsilon}\left[\LCE(h^\text{s}(\tx), y) + \lambda \left\|\tx -  \arg\max_{\|x' - \tx\| \leq \epsilon_{\text{defense}}} \LDB(x')\right\|,\right]\]
    where $\LCE$ is the cross-entropy loss, $h^\text{s}$ returns the softmax activations of $h$ and where\\
    $\LDB(x) = \text{rank}_2 h^\text{s}(x) - \max h^\text{s}(x)$.\\
    
    \COMMENT{If the attack did not succeed against $h$ (in other words, if $h$ does not robustly predict $\hat{y} \neq y$), check whether to leave $x$ unperturbed.}
    \STATE \[x' =  \arg\max_{\|x' - \tx\| \leq \epsilon_{\text{defense}}} \LCE(h^s(x'), h(\tx))\]
    \IF{$h(x') \neq h(\tx) \lor h(\tx) = y$}
        \STATE Leave $x$ unperturbed if $F(h)$ rejects it, or if $h(x) \neq y$.
        \STATE \[x'' = \arg\max_{\|x'' - x\| \leq \epsilon_{\text{defense}}} \LCE(h^s(x''), h(x))\]
        \IF{$h(x) \neq y \lor h(x'') \neq h(x)$}
            \STATE $\tx = x$
        \ENDIF 
    \ENDIF
    \STATE \textbf{Return:} $\tx$
  \end{algorithmic}
  \label{alg:rej-attack-transfer}
\end{algorithm*}

Finally, note the attack with $\LREJ$, without GMSA, is effective against selective classifiers based on the transformation $F$ (and via Tram\`er's equivalency, selective classifiers in general). So we summarize this attack on a fixed model in Algorithm~\ref{alg:rej-attack}.

\begin{algorithm*}[htb]
  \small
  \caption{\textsc{\small Inductive Rejection-Aware Attack}}
  \begin{algorithmic}[1]
    \REQUIRE A model $h$, and a clean labelled test point $(x,y)$, an adversarial budget of $\epsilon$, and a radius used for rejection $\epsilon_{\text{defense}}$.
    \STATE Search for a perturbation $\tx$ of $x$ for which $h$ predicts $\hat{y} \neq y$ robustly.
    \[\tx = \arg\max_{\|\tx - x\| \leq \epsilon}\left[\LCE(h^{s}(\tx), y) + \lambda \left\|\tx -  \arg\max_{\|x' - \tx\| \leq \epsilon_{\text{defense}}} \LDB(x')\right\|\,\right]\]
    where $\LCE$ is the cross-entropy loss, $h^{s}$ returns the softmax activations of $h$ and where\\ $\LDB(x') = \text{rank}_2 h^{s}(x') - \max h^{s}(x')$
    \STATE \textbf{Return:} $\tx$
  \end{algorithmic}
  \label{alg:rej-attack}
\end{algorithm*}

\subsection{Rejectron Experiments}\label{app:rejectron-experiments}

Goldwasser et al.'s implementation of Rejectron~\citep{goldwasser2020beyond} trains a classifier (call it $h_c$) on the training set and a discriminator ($h_d$) to distinguish between the (clean) training and (potentially-perturbed) test data. Samples are rejected if the discriminator classifies them as test data; otherwise, the classifier's prediction is returned. Our adaptive attack is then very simple: we follow the approach of \Cref{alg:rej-attack-gmsa} but with a loss function $\LDISC$ which targets the defense.

Given a sample $(x,y)$, the attacker's goal is to flip the label, and, simultaneously, to avoid rejection; hence, we maximize the following loss:
\[\LDISC(x, y) = \LCE(h_c^{\text{s}}(x), y) + \lambda \LCE(h_d^{\text{s}}(x), 1)\]

where class 1 for $h_d$ corresponds to test data, signalling rejection, and where $h^{\text{s}}$ returns the softmax activations of $h$. Maximizing $\LDISC$ then minimizes the confidence in the true label and the probability of rejection.

\begin{figure}[ht]
\begin{multicols}{2}
  \begin{center}
    \includegraphics[width=0.9\columnwidth]{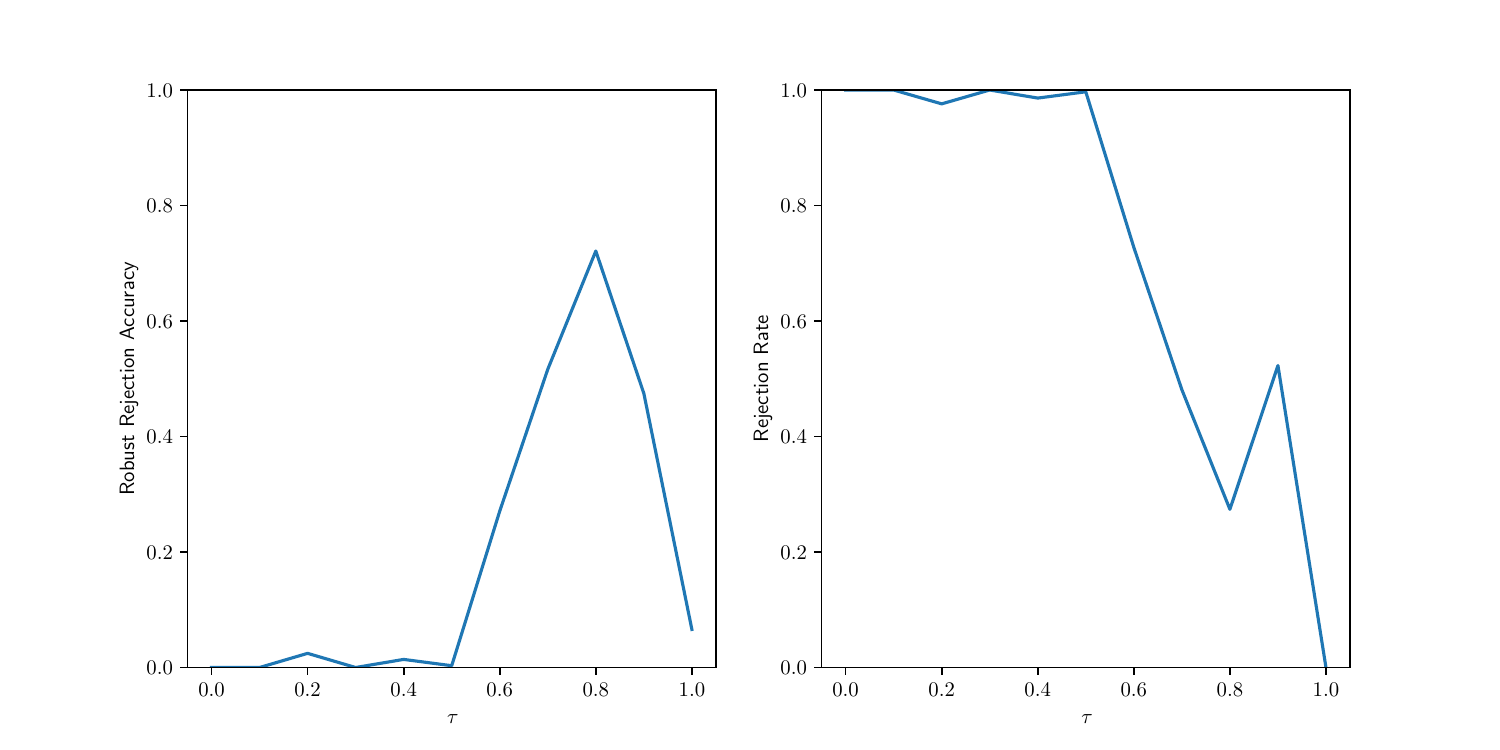}
    \caption{Effects of $\tau$ on performance of Rejectron on MNIST with attacker GMSA ($\LDISC$).}
    \label{fig:rejectron-mnist}
   \end{center}
  \columnbreak
  \begin{center}
    \includegraphics[width=0.9\columnwidth]{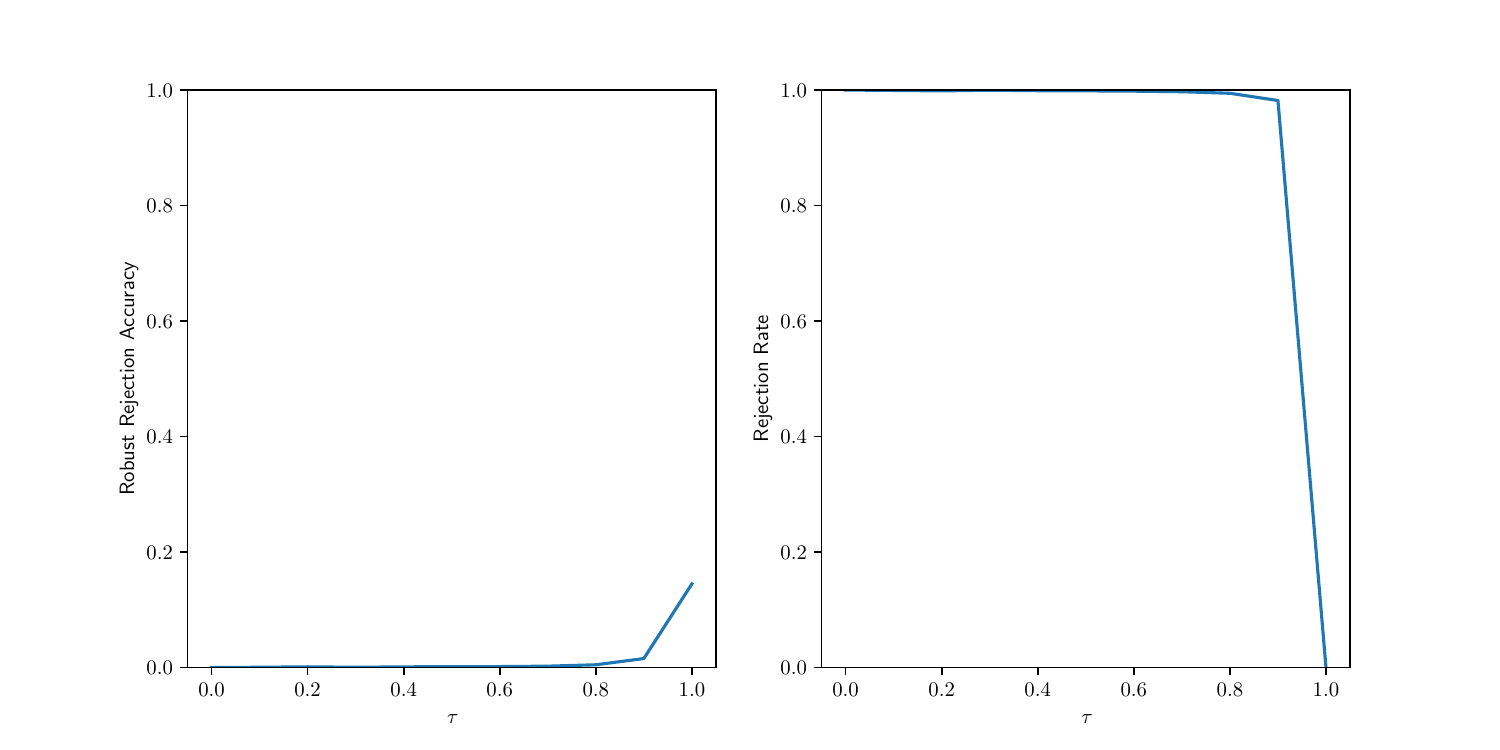}
    \caption{Effects of $\tau$ on performance of Rejectron on CIFAR-10 with attacker GMSA ($\LDISC$).}
    \label{fig:rejectron-cifar10}
   \end{center}
\end{multicols}
\end{figure}
 
Figures~\ref{fig:rejectron-mnist} and~\ref{fig:rejectron-cifar10} show our adaptive attack's performance on MNIST and CIFAR-10. $\tau$ is a key hyperparameter of Rejectron, which determines the confidence needed by $h_d$ to reject a sample; to evaluate Rejectron fairly, we report the results on best-performing value of $\tau$, based on (transductive) robust rejection accuracy; see Table~\ref{tab:result}. On CIFAR-10, performance is near-zero and rejection rate is near 100\% for small values of $\tau$. The best-performing value of $\tau$ is 1 (effectively eliminating the possibility of rejection), leading to a rejection rate of 0; this behavior on CIFAR-10 illustrates the algorithm's struggles with the practical high-complexity deep learning setting.

\section{Additional Experiments}\label{app:attack-ablation}

\begin{table*}[!t]
\begin{center}
    \begin{tabular}{cc c cc cc}
    \toprule
    \multicolumn{2}{c }{TLDR Components} & \multirow{2}{*}{Attacker} & \multicolumn{2}{c }{MNIST} & \multicolumn{2}{c}{CIFAR-10}\\ \cline{1-2} \cline{4-7}
    Rejection & $L_{\text{test}}$ & & $\PREJ$ & Robust accuracy & $\PREJ$ & Robust accuracy
    \\
    \midrule
    \checkmark & \checkmark & GMSA ($\LREJ$) &  \textbf{0.588} & 0.967 & 0.208 & \textbf{0.739}\\
    \midrule
    \checkmark & $\times$ & GMSA ($\LREJ$) & 0.646 & \textbf{0.975} & \textbf{0.179} & 0.725 \\
    \midrule
    $\times$ & \checkmark & GMSA ($\LCE$)  & -- & 0.900 & -- & 0.516\\
    \midrule
    $\times$ & $\times$ & GMSA ($\LCE$) & -- & 0.935 & -- & 0.516\\
    \bottomrule
    \end{tabular} 
    \caption{Ablation study of TLDR. The best result is \textbf{boldfaced}.}
    \label{tab:ablation-defense}
\end{center}
\end{table*}

\subsection{Ablation Study of TLDR}\label{sec:ablation}

Compared to traditional defenses, TLDR has two novel components: using the given test inputs in training the classifier (the second term in \Cref{eqn:transductive-loss}, referred to as $L_\text{test}$), and transforming the trained classifier into one with rejection. \Cref{tab:ablation-defense} shows the results of the ablation study on these two components. In all cases, rejection significantly improves results. The use of transduction is helpful on CIFAR-10, but reduces performance on MNIST, potentially due to the lower difficulty of deriving robust predictions on MNIST; hence, the knowledge of the specific test inputs is less useful.

\subsection{Warm Start in TLDR}

\begin{table}[!t]
  \caption{\centering Effects of warm start period on TLDR.}
    \vskip 0.1in
\begin{center}
\begin{tabular}{l|c|c}
    Warm start (epochs) & Rejection Rate & Robust Rejection Accuracy \\
    \hline
    0 & 0.813 & 0.153\\
    \textbf{500} & \textbf{0.531} & \textbf{0.177}\\
    1000 & 0.830 & 0.171\\
\end{tabular}
\end{center}
\vskip -0.1in
\end{table}

Here we perform experiments showing that in training TLDR, it is best to first trains a baseline model without transductive regularization $L_{\textrm{test}}$ in the early stage (warm start) and then add transductive regularization for later training. 

We generate the data with 100  Gaussians (one per class) equally spaced in $l_\infty$ with a separation of 3 units between means. The adversarial budget is 2 units, and we ensure that the data is sparse by generating 10 samples per class. The models are 10 layer feedforward networks with skip connections.

The synthetic models are trained for 1000 epochs total; we see the best performance when the model has transductive regularization but is allowed to learn an initial baseline model before transductive regularization is used in training. Doing so reduces the risk of the regularization term harming performance.

\subsection{GMSA Method}

\begin{table*}[ht]
    \caption{\centering Full ablation results of TLDR.} 
    \vskip 0.1in
    \begin{adjustbox}{width=\columnwidth,center}
    \begin{tabular}{cc|l|cc|cc}
    \toprule
    \multicolumn{2}{c|}{TLDR Components} & \multirow{2}{*}{Attacker} & \multicolumn{2}{c|}{MNIST} & \multicolumn{2}{c}{CIFAR-10}\\ \cline{4-7}
    Rejection & Transductive Regularization & & $\PREJ$ & Robust accuracy & $\PREJ$ & Robust accuracy
    \\
    \midrule
    \checkmark & \checkmark & $\text{GMSA}_\text{AVG}$ ($\LREJ$) & 0.796 & 0.968 & 0.195 & 0.744\\
    \checkmark & \checkmark & $\text{GMSA}_\text{MIN}$ ($\LREJ$) & 0.588 & 0.967 & 0.208 & 0.739\\
    \checkmark & $\times$ & $\text{GMSA}_\text{AVG}$ ($\LREJ$) & 0.646 & 0.975 & 0.179 & 0.725 \\
    \checkmark & $\times$ & $\text{GMSA}_\text{MIN}$ ($\LREJ$) & 0.202 & 0.980 & 0.182 & 0.733 \\
    $\times$ & \checkmark & $\text{GMSA}_\text{AVG}$ ($\LCE$) & -- & 0.900 & -- & 0.516\\
    $\times$ & \checkmark & $\text{GMSA}_\text{MIN}$ ($\LCE$) & -- & 0.914 & -- & 0.601\\
    $\times$ & $\times$ & $\text{GMSA}_\text{AVG}$ ($\LCE$) & -- & 0.935 & -- & 0.516\\
    $\times$ & $\times$ & $\text{GMSA}_\text{MIN}$ ($\LCE$) & -- & 0.942 & -- & 0.556\\
    \bottomrule
    \end{tabular} 
    \end{adjustbox}
    \smallskip
    \label{tab:ablation-defense-full}
    \vskip -0.1in
\end{table*}

We present extended results of our defense ablation and compare the results of $\text{GMSA}_\text{AVG}$, which optimizes the average loss of past iterations, and $\text{GMSA}_\text{MIN}$, which optimizes the worst-case loss. See \citep{chen2022towards}. We can see that while the two perform about the same on the full TLDR defense ($\text{GMSA}_\text{MIN}$ performs slightly better), $\text{GMSA}_\text{AVG}$ is much stronger for models not incorporating both components.

\subsection{Rejection Radius} \label{app:rejection-radius}

\begin{figure}[h]
    \centering
    \includegraphics[width=0.9\columnwidth]{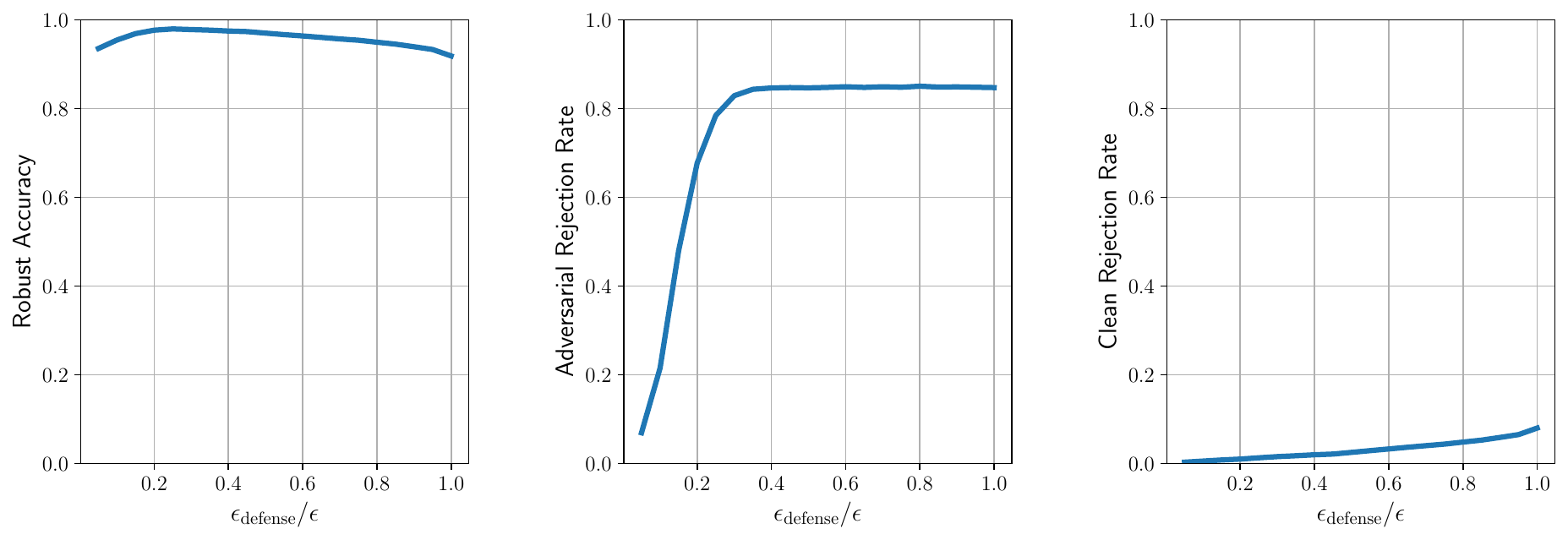}
    \caption{Effects of rejection radius $\epsilon_{\text{defense}}$ on MNIST (inductive) with attacker PGD ($\LREJ$).}
    \label{fig:rejection-radius-inductive}
\end{figure}

\begin{figure}[h]
    \centering
    \includegraphics[width=0.9\columnwidth]{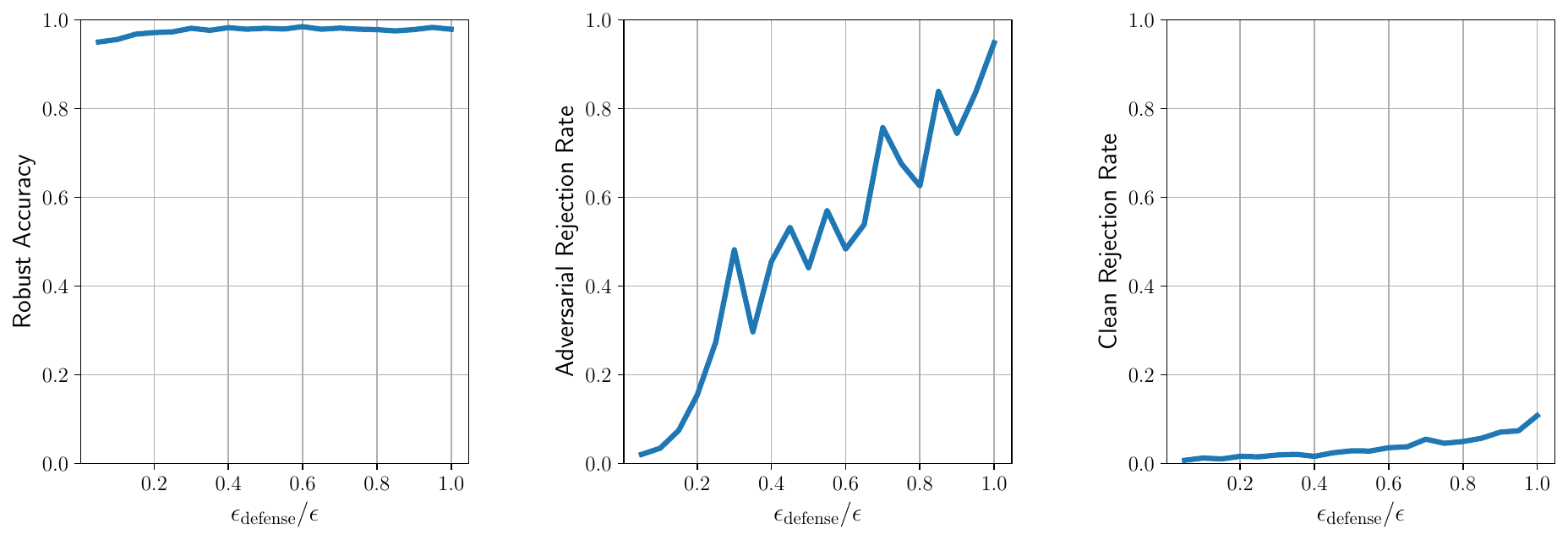}
    \caption{Effects of rejection radius $\epsilon_{\text{defense}}$ on MNIST (TLDR) with attacker GMSA ($\LREJ$).}
    \label{fig:rejection-radius-tldr}
\end{figure}

The rejection radius $\epsilon_{\text{defense}}$ is an important hyper-parameter for TLDR; however, the model's performance is not very sensitive to it. Figure~\ref{fig:rejection-radius-inductive} shows the trend of robust accuracy, the rejection rate on adversarial test data, and the rejection rate on clean test data, for the inductive classifier on MNIST; \Cref{fig:rejection-radius-tldr} shows those for TLDR. The robust accuracy remains stable. The theoretical analysis suggests setting the radius to $\epsilon / 3$ where $\epsilon$ is the adversarial budget. Given TLDR's low sensitivity to the parameter, we use $\epsilon/4$ for consistency as the inductive case performs best with that setting. The rejection rate on the adversarial test data rises rapidly with the rejection radius (reaching 0.949 for TLDR for $\epsilon_{\textrm{defense}} = \epsilon$), but the rejection rate on clean data increases much more slowly (0.108 when $\epsilon_{\textrm{defense}} = \epsilon$). 
So among all rejected inputs only a few are clean inputs, leading to low errors as desired. 

The rejection rate on clean inputs is presented for the transductive case in order to illustrate the difference in effects on clean and perturbed data, but, as the adversary may select to perturb, some clean points were not in the training set, and, hence, the clean rejection rates should not be considered reliable.
The rejection rates rise with the rejection radius: adversarial rejection rates increase rapidly as the rejection radius increases, while clean rejection rates increase only slowly. In all cases, far more perturbed samples are rejected than clean samples.

\subsection{Binarization test on PGD ($\LREJ$)}

\begin{table*}[ht]
  \caption{Results of the binarization test applied to PGD ($\LREJ$).}
  \vskip 0.1in
  \begin{adjustbox}{width=\columnwidth,center}
    \begin{tabular}{l cccc cccc}
        \toprule
        & \multicolumn{4}{c}{MNIST} & \multicolumn{4}{c}{CIFAR-10}\\
        \cline{2-5}\cline{6-9}\\
        Decision Boundary Closeness & ASR & RASR & Inverted ASR & Inverted RASR & ASR & RASR & Inverted ASR & Inverted RASR\\
        \midrule
        0.9 & 0.935 & 0.451 & 1.0 & 0.375 & 0.973 & 0.824 & 0.971 & 0.781\\
        \midrule
        0.999 & 0.945 & 0.394 & 1.0 & 0.447 & 0.976 & 0.813 & 0.964 & 0.790\\
        \midrule
        0.99999 & 0.953 & 0.414 & 0.981 & 0.434 & 0.974 & 0.819 & 0.938 & 0.813\\
        \bottomrule
    \end{tabular}
  \end{adjustbox}
  \label{tab:binarization}
  \vskip -0.1in
\end{table*}

Finally, to evaluate the effectiveness with which $\LREJ$ targets rejection, we apply the binarization test~\citep{binarization}. As the binarization test is designed for inductive defenses we evaluate on PGD ($\LREJ$), and as the binarization test assumes that rejection does not depend on the generated dataset or the modified model, we modified $\LREJ$ to target the original model in the calcuation of $\LDB$ (e.g. we wish to avoid rejection as if the model was unchanged). 

For the inverted case, we modified $\lambda’$, setting it to -10 (we are seeking rejection, not avoiding it). As noted in Appendix~\ref{app:attack-ablation}
, we drop the rejection-avoidance term when $h(\tilde{x}) = y$; hence, the negated second term poses issues for maximization in PGD (e.g. PGD would preferentially select perturbations which do not succeed). To avoid this issue, we have added an additional success indicator to our attack objective, which we use to ensure that PGD selects the loss-maximizing successful perturbation. Without these modifications, we observed low attack success rates in the inverted test; however, the results with these simple changes do indicate that our attack does take the rejection component of the defense into account, the key purpose of the inverted test.

The attack settings for the regular test are unchanged from those used for evaluation. For the test settings, we chose values as close as possible to those used in~\citep{binarization}, with a single boundary sample, with 200 samples sampled from each of the surfaces and corners of the $l_{\infty}$ ball, with 512 trials per experiment. We used 81 inner samples for MNIST, and 253 inner samples for CIFAR-10, selected to maximize subject to the requirement that the total sample count is below the dimensionality of the features. In both cases, the base model is a standard adversarially trained model trained on that dataset, transformed into a selective classifier with the transformation $\operatorname{F}$. 

We ran the test for a range of values for the decision boundary closeness, a hyperparameter determining the test hardness. ASR is the rate at which the attack successfuly found a perturbation which both flips the label and evades detection; RASR is the maximum of the success rates on surfaces and corners. 
While the ASR values in some experiments are slightly below the cutoff of 0.95 and are technically failures, they do indicate that the attack is successfully targeting the defense. While a slightly stronger attack may exist, these results do not indicate significant unreliability in our evaluation of the robustness of TLDR.
 
\subsection{Ablation on Attacks: Attack Radius}

The theory suggests that incorporating rejection can allow a transductive learner to tolerate perturbations twice as large; we investigate how transduction and rejection affects the robustness as $\epsilon$ grows  (models are adversarially trained with the corresponding $\epsilon$ and the selective classifiers use a rejection radius of $\epsilon / 2$). The results are shown for the natural choice of adversary, as in the experiment section (e.g. GMSA with $\LREJ$ for the transduction+rejection). For selective classifiers, the rejection rate scaling is shown.

We see that the combination of rejection and transduction does indeed maintain high accuracy for larger $\epsilon$; at $\epsilon = 0.6$, it has 96.2\% of the robust accuracy that transduction alone had for $\epsilon = 0.3$. This aligns with the theory, given the increased constant factors of $\text{OPT}_{\U^2}$ in Corollary \ref{corr:transductive-agnostic} compared to the results for classifiers in \citep{montasser2021transductive}.

Note also the behavior of the inductive classifier: accuracy improves past $\epsilon = 0.6$. To see why, note that a model adversarially trained for $\epsilon \ge 1$ will return near-uniform predictions for all classes (resulting in a robust accuracy of approximately 10\%, as seen), making finding adversarial examples slightly more difficult than for smaller $\epsilon$ where this does not occur. The decline in rejection rate for very large $\epsilon$ is a similar phenomenon.

\begin{figure}[t]
    \centering
    \includegraphics[width=0.9\linewidth]{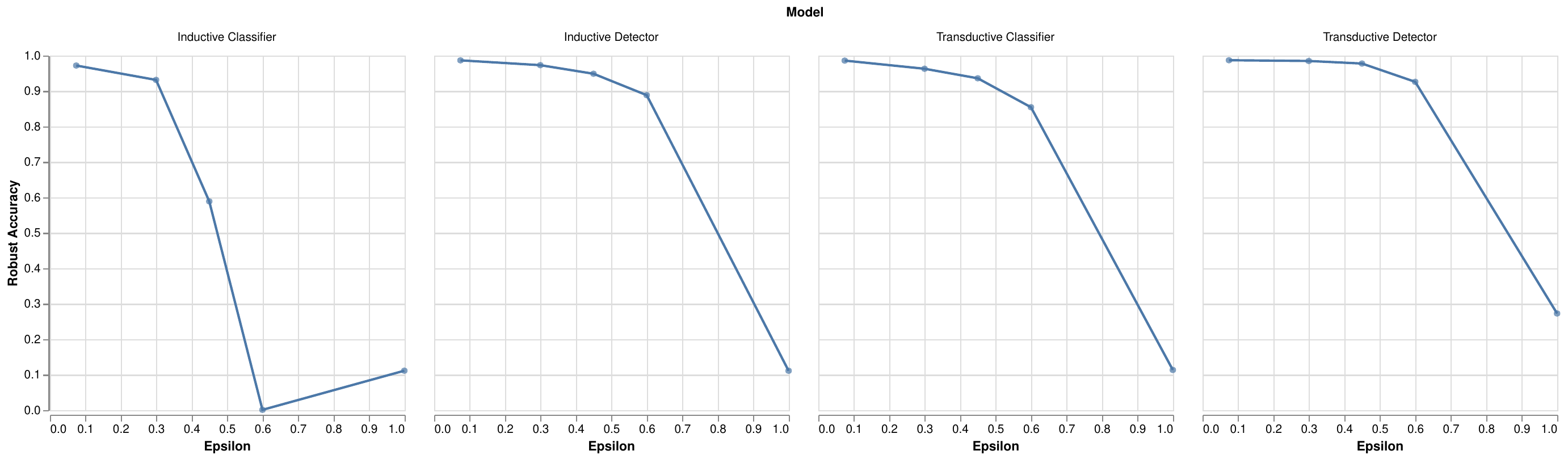}
    \caption{Robustness scaling with adversarial budget $\epsilon$ on MNIST}
    \label{fig:epsilon-accuracy}
\end{figure}
\begin{figure}[t]
    \centering
    \includegraphics[width=0.5\linewidth]{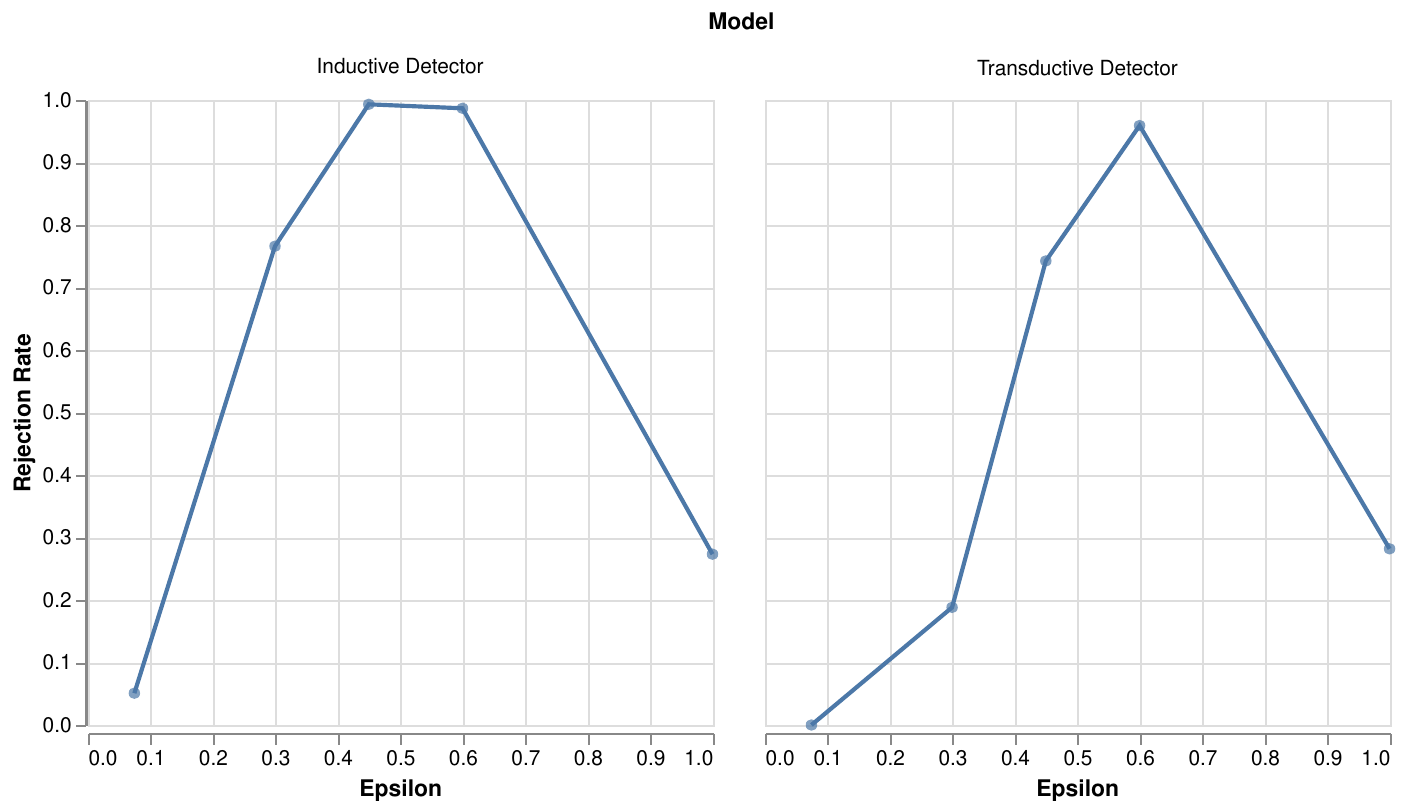}
    \caption{Rejection rate scaling with adversarial budget $\epsilon$ on MNIST.}
    \label{fig:epsilon-rejection}
\end{figure}

\subsection{Weighting of $\LREJ$}

\begin{figure}[t]
    \centering
    \includegraphics[width=0.5\linewidth]{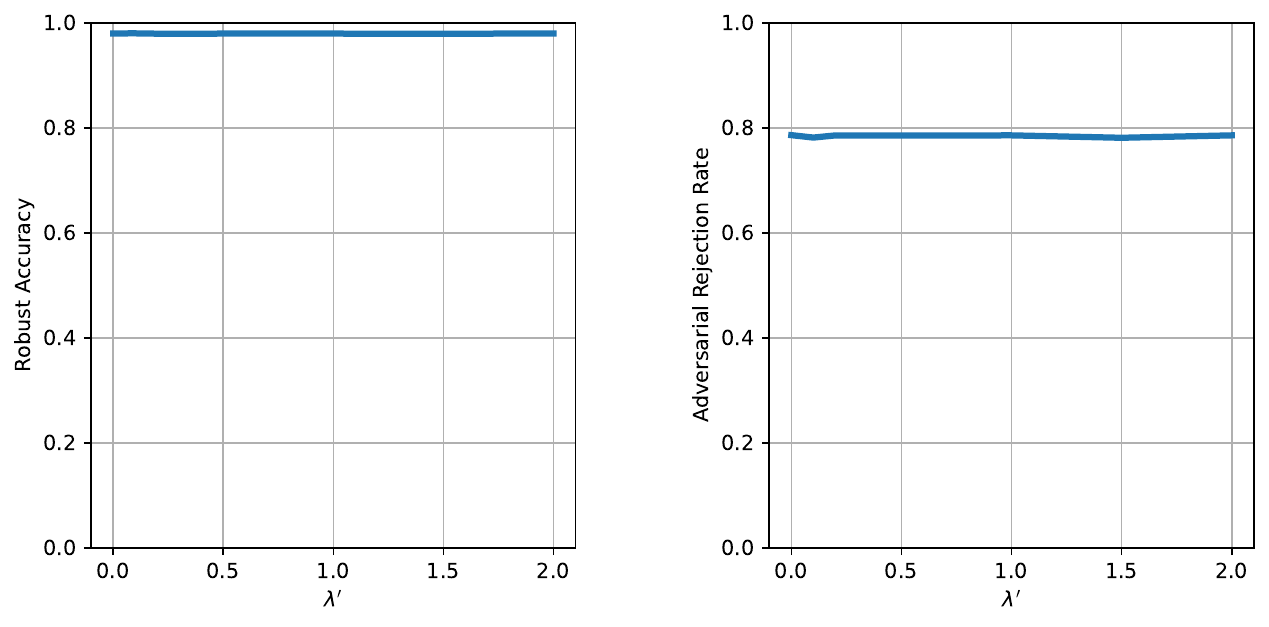}
    \caption{Effects of $\lambda'$ on results of PGD optimizing $\LREJ$ targeting adversarial training with rejection on MNIST.}
    \label{fig:lambda-attack-inductive}
\end{figure}
\begin{figure}[ht]
    \centering
    \includegraphics[width=0.5\linewidth]{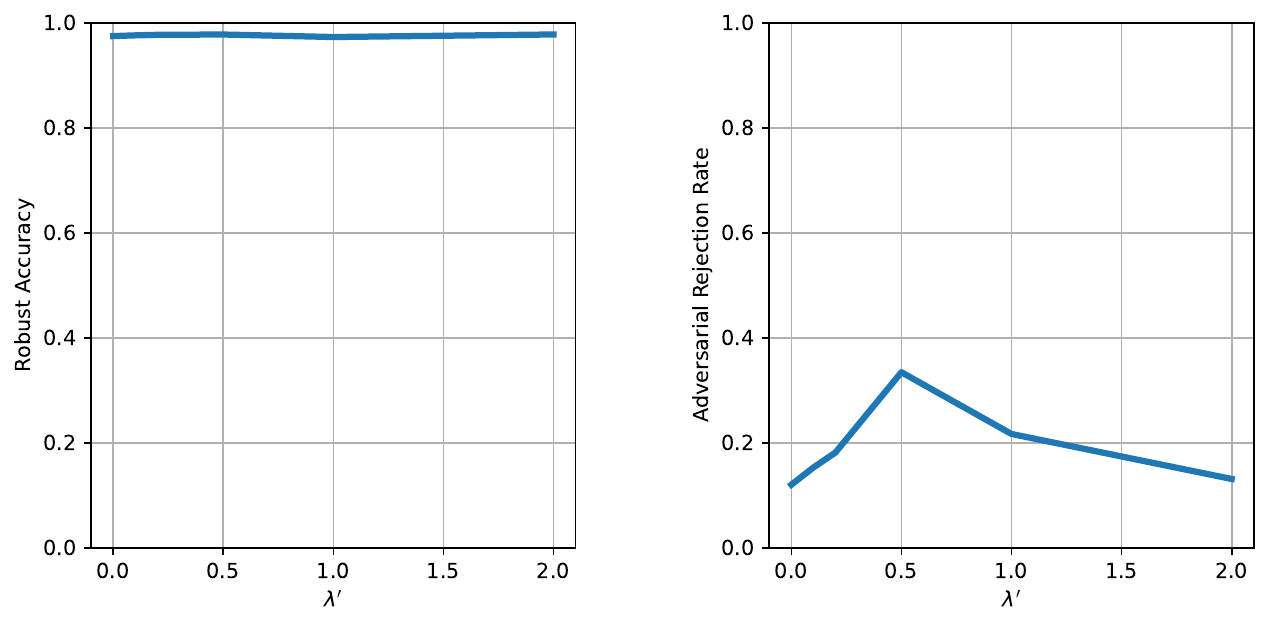}
    \caption{Effects of $\lambda'$ on results of GMSA optimizing $\LREJ$ targeting TLDR on MNIST.}
    \label{fig:lambda-attack-transductive}
\end{figure}

We examine the effect of the hyperparameter $\lambda'$ between the cross-entropy and rejection-avoidance terms in $\LREJ$ on MNIST; see Equation~\ref{eqn:lrej}. In the inductive case, as shown in Figure~\ref{fig:lambda-attack-inductive}, there is little sensitivity to $\lambda'$ in either attack success rate or rejection rate. When targeting TLDR, there is little sensitivity in terms of attack success rate as seen in Figure~\ref{fig:lambda-attack-transductive}; rejection rate is highest for intermediate values of $\lambda'$ but, as expected, rejection rate declines with $\lambda'$ beyond that.

\subsection{Robustness to $l_2$}

\begin{table*}[ht]
    \caption{Results on MNIST and CIFAR-10 up to $l_2$ budget. The strongest attack against each defense is shown. 
    The best result is  \textbf{boldfaced}.}
    \vskip 0.1in
    \begin{adjustbox}{width=\columnwidth,center}
    \begin{tabular}{l l l cc cc}
    \toprule
    \multirow{2}{*}{Setting} & \multirow{2}{*}{Defense} & \multirow{2}{*}{Attacker} & \multicolumn{2}{c }{MNIST} & \multicolumn{2}{c}{CIFAR-10}\\ \cline{4-5} \cline{6-7}
    & & & $\PREJ$ & Robust accuracy & $\PREJ$ & Robust accuracy
    \\
    \midrule
    Induction & AT~\citep{madry2018towards} & AutoAttack & -- & 0 & -- & 0.445 \\
    \midrule
    Rejection only & AT (with rejection) & PGD ($\LREJ$) & 0.112 & 0.921 & 0.130 & 0.754 \\
    \midrule
    Transduction only & TADV~\citep{chen2022towards}  & GMSA ($\LCE$) & -- & 0.913 & -- & 0.813\\
    \midrule 
    Transduction+Rejection & \textbf{TLDR (ours)}  & GMSA ($\LREJ$) & \textbf{0.078} & \textbf{0.933} &  \textbf{0.007} & \textbf{0.845}\\
    \bottomrule
    \end{tabular} 
    \end{adjustbox}
    \label{tab:result_l2}
    \vskip -0.1in
\end{table*}

To evaluate our defense's generality, we consider robustness to $l_2$ as well and compare to the strongest defenses from each setting in Table~\ref{tab:result_l2}; on MNIST we use $\epsilon = 5$ and on CIFAR-10 we use $\epsilon = 128/255$. We observe strong performance from TLDR, outperforming defenses with transduction or rejection alone.

\subsection{Generalization of TLDR}

\begin{figure}[t]
    \centering
    \includegraphics[width=0.9\linewidth]{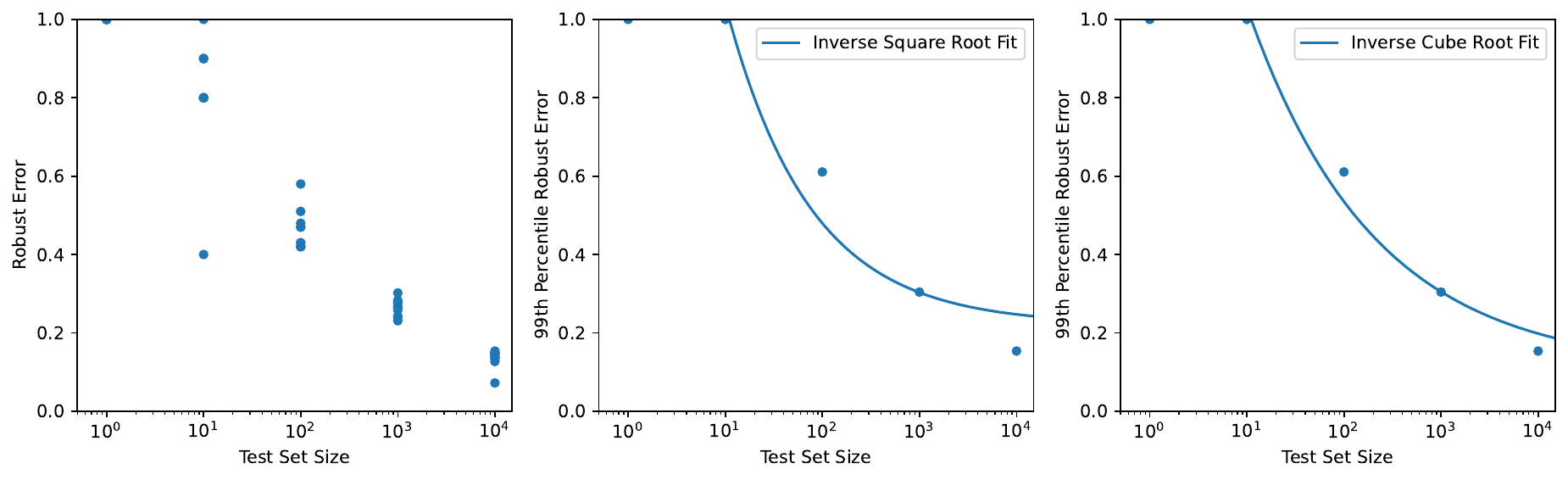}
    \caption{Generalization of TLDR with equal train and test size on MNIST.}
    \label{fig:balanced_train_test_generalization}
\end{figure}
\begin{figure}[t]
    \centering
    \includegraphics[width=0.9\linewidth]{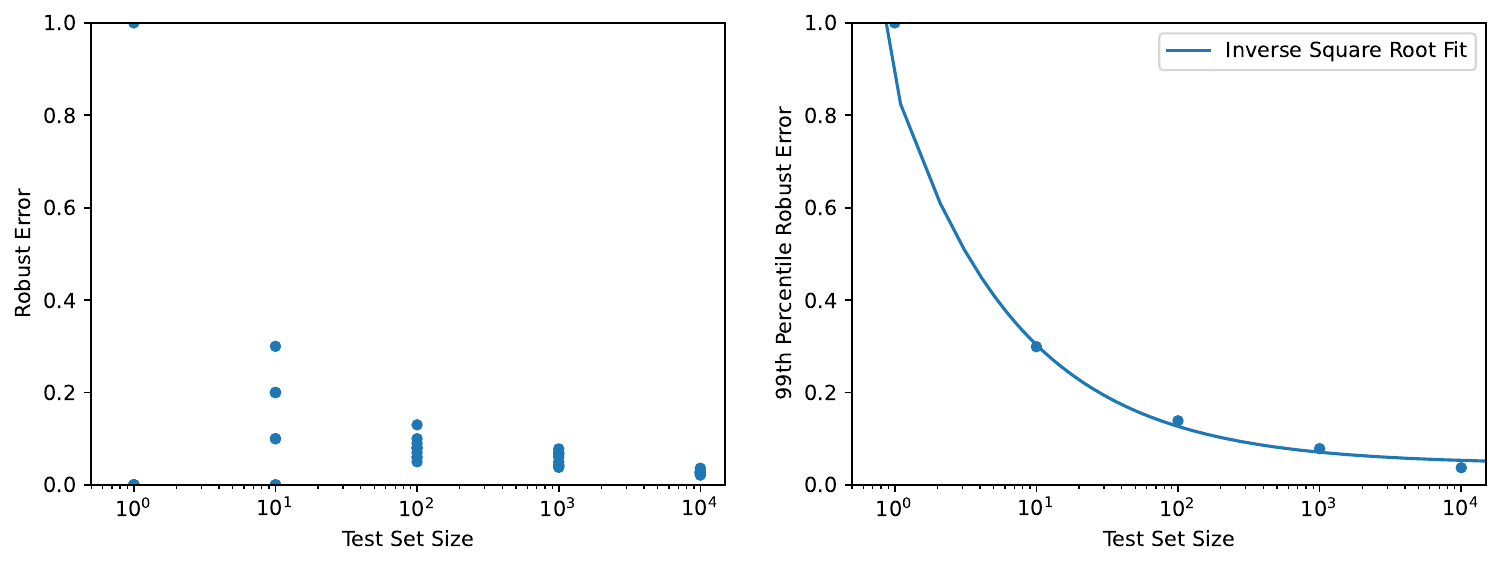}
    \caption{Generalization of TLDR with full training set on MNIST.}
    \label{fig:full_train_test_generalization}
\end{figure}

To evaluate how closely TLDR's generalization follows the our provided bounds in Theorems~\ref{thm:transductive-realizable} and~\ref{thm:transduction-agnostic}, we apply TLDR to randomly-sampled subsets of the MNIST training and test sets. In each case, we run ten trials and present the robust error (1 - robust accuracy) with attacker GMSA ($\LREJ$). Given the large VC dimension of the model considered (LeNet)~\citep{bartlett2017nearlytight}, the results shown are consistent with Theorem~\ref{thm:transduction-agnostic}; we wish to determine whether the actual errors observed follow the inverse-square relationship of the theorem.

In Figure~\ref{fig:balanced_train_test_generalization}, the size of the training set is set equal to the size of the test set (the standard assumption for our results); in Figure~\ref{fig:full_train_test_generalization}, the full training set is used and only the test set size is changed. See Appendix~\ref{app:unbalanced-train-test} for a discussion of generalization bounds for train and test sets of differing sizes.

As the bounds are in PAC form, we use an estimate of the 99th percentile of error in order to evaluate the generalization of TLDR; these are calculated with a best-fit beta distribution of the results on each instance size.

We then consider the inverse-square-root fit of these 99th percentile error estimates; as the gurarantee takes the form of an upper bound, and error is upper bounded by 1, we exclude any error values equal to 1 (corresponding to instances where all trials had a robust accuracy of 0).
We find that in the case where train size is fixed, the 99th percentile errors closely follow the inverse-square-root trend in the test set size $m$; while the results for equal train and test set sizes more closely follow an inverse-cube-root relationship in $m$.

\section{Limitations}\label{app:limitations}
While our framework is theoretical-sound with lower sampled complexity than the rejection-only case and with more relaxed optimality condition than the transductive-only case, our sample complexity proof under the transductive rejection case requires the non-emptiness of $\Delta$ in \Cref{thm:rejection-simplified-realizable}. While weaker conditions don’t guarantee that we find a
model satisfying the conditions, the result demonstrate that empirical defense incorporating both transduction and rejection have the potential to outperform others. Our proposed defense algorithm TLDR, though effective at improving the robust accuracy under rejection, incurs a high computational cost relative to standard adversarial training due to the joint training with the unlabeled data. If it is possible to delay evaluation until a sufficiently large batch of samples arrives, the cost can be made insignificant via amortization. The need to perform a full training process prior to evaluation means, however, that the defense is not suitable for latency-sensitive applications. Our adaptive attack is even more costly, as effectively attacking this defense using GMSA requires multiple iterations of the full transductive training process; hence, adversaries attacking TLDR require substantial resources.


\begin{thebibliography}{40}
\providecommand{\natexlab}[1]{#1}
\providecommand{\url}[1]{\texttt{#1}}
\expandafter\ifx\csname urlstyle\endcsname\relax
  \providecommand{\doi}[1]{doi: #1}\else
  \providecommand{\doi}{doi: \begingroup \urlstyle{rm}\Url}\fi

\bibitem[Assouad(1983)]{assouad1983densite}
Assouad, P.
\newblock Densit{\'e} et dimension.
\newblock In \emph{Annales de l'Institut Fourier}, volume~33, pp.\  233--282,
  1983.

\bibitem[Baharlouei et~al.(2022)Baharlouei, Sheikholeslami, Razaviyayn, and
  Kolter]{baharlouei2022improving}
Baharlouei, S., Sheikholeslami, F., Razaviyayn, M., and Kolter, Z.
\newblock Improving adversarial robustness via joint classification and
  multiple explicit detection classes.
\newblock \emph{arXiv preprint arXiv:2210.14410}, 2022.

\bibitem[Bartlett et~al.(2017)Bartlett, Harvey, Liaw, and
  Mehrabian]{bartlett2017nearlytight}
Bartlett, P.~L., Harvey, N., Liaw, C., and Mehrabian, A.
\newblock Nearly-tight vc-dimension and pseudodimension bounds for piecewise
  linear neural networks, 2017.

\bibitem[Blumer et~al.(1989)Blumer, Ehrenfeucht, Haussler, and
  Warmuth]{blumer1989learnability}
Blumer, A., Ehrenfeucht, A., Haussler, D., and Warmuth, M.~K.
\newblock Learnability and the vapnik-chervonenkis dimension.
\newblock \emph{Journal of the ACM (JACM)}, 36\penalty0 (4):\penalty0 929--965,
  1989.

\bibitem[Carlini \& Wagner(2017)Carlini and Wagner]{carlini2017towards}
Carlini, N. and Wagner, D.
\newblock Towards evaluating the robustness of neural networks.
\newblock In \emph{2017 ieee symposium on security and privacy (sp)}, pp.\
  39--57. Ieee, 2017.

\bibitem[Chen et~al.(2021)Chen, Raghuram, Choi, Wu, Liang, and
  Jha]{chen2021revisiting}
Chen, J., Raghuram, J., Choi, J., Wu, X., Liang, Y., and Jha, S.
\newblock Revisiting adversarial robustness of classifiers with a reject
  option.
\newblock In \emph{The AAAI-22 Workshop on Adversarial Machine Learning and
  Beyond}, 2021.

\bibitem[Chen et~al.(2022)Chen, Wu, Guo, Liang, and Jha]{chen2022towards}
Chen, J., Wu, X., Guo, Y., Liang, Y., and Jha, S.
\newblock Towards evaluating the robustness of neural networks learned by
  transduction.
\newblock In \emph{International Conference on Learning Representations}, 2022.
\newblock URL \url{https://openreview.net/forum?id=_5js_8uTrx1}.

\bibitem[Croce \& Hein(2020)Croce and Hein]{croce2020reliable}
Croce, F. and Hein, M.
\newblock Reliable evaluation of adversarial robustness with an ensemble of
  diverse parameter-free attacks.
\newblock In \emph{International conference on machine learning}, pp.\
  2206--2216. PMLR, 2020.

\bibitem[Croce et~al.(2020)Croce, Andriushchenko, Sehwag, Debenedetti,
  Flammarion, Chiang, Mittal, and Hein]{croce2020robustbench}
Croce, F., Andriushchenko, M., Sehwag, V., Debenedetti, E., Flammarion, N.,
  Chiang, M., Mittal, P., and Hein, M.
\newblock Robustbench: a standardized adversarial robustness benchmark.
\newblock \emph{arXiv preprint arXiv:2010.09670}, 2020.

\bibitem[Ganin et~al.(2016)Ganin, Ustinova, Ajakan, Germain, Larochelle,
  Laviolette, Marchand, and Lempitsky]{domainadversarial}
Ganin, Y., Ustinova, E., Ajakan, H., Germain, P., Larochelle, H., Laviolette,
  F., Marchand, M., and Lempitsky, V.
\newblock Domain-adversarial training of neural networks.
\newblock \emph{J. Mach. Learn. Res.}, 17\penalty0 (1):\penalty0 2096–2030,
  jan 2016.
\newblock ISSN 1532-4435.

\bibitem[Goldwasser et~al.(2020)Goldwasser, Kalai, Kalai, and
  Montasser]{goldwasser2020beyond}
Goldwasser, S., Kalai, A.~T., Kalai, Y., and Montasser, O.
\newblock Beyond perturbations: Learning guarantees with arbitrary adversarial
  test examples.
\newblock \emph{Advances in Neural Information Processing Systems},
  33:\penalty0 15859--15870, 2020.

\bibitem[Goodfellow(2019)]{goodfellow2019research}
Goodfellow, I.
\newblock A research agenda: Dynamic models to defend against correlated
  attacks.
\newblock \emph{arXiv preprint arXiv:1903.06293}, 2019.

\bibitem[Goodfellow et~al.(2015)Goodfellow, Shlens, and
  Szegedy]{goodfellow2015explaining}
Goodfellow, I.~J., Shlens, J., and Szegedy, C.
\newblock Explaining and harnessing adversarial examples, 2015.

\bibitem[Hanneke et~al.(2019)Hanneke, Kontorovich, and
  Sadigurschi]{hanneke2019sample}
Hanneke, S., Kontorovich, A., and Sadigurschi, M.
\newblock Sample compression for real-valued learners.
\newblock In \emph{Algorithmic Learning Theory}, pp.\  466--488. PMLR, 2019.

\bibitem[He et~al.(2022)He, Yang, Chen, Xu, and Ho]{he2022your}
He, Z., Yang, Y., Chen, P.-Y., Xu, Q., and Ho, T.-Y.
\newblock Be your own neighborhood: Detecting adversarial example by the
  neighborhood relations built on self-supervised learning.
\newblock \emph{arXiv preprint arXiv:2209.00005}, 2022.

\bibitem[Kato et~al.(2020)Kato, Cui, and Fukuhara]{kato2020atro}
Kato, M., Cui, Z., and Fukuhara, Y.
\newblock Atro: Adversarial training with a rejection option.
\newblock \emph{arXiv preprint arXiv:2010.12905}, 2020.

\bibitem[Krizhevsky et~al.(2009)Krizhevsky, Hinton,
  et~al.]{krizhevsky2009learning}
Krizhevsky, A., Hinton, G., et~al.
\newblock Learning multiple layers of features from tiny images.
\newblock 2009.

\bibitem[Laidlaw \& Feizi(2019)Laidlaw and Feizi]{laidlaw2019playing}
Laidlaw, C. and Feizi, S.
\newblock Playing it safe: Adversarial robustness with an abstain option.
\newblock \emph{arXiv preprint arXiv:1911.11253}, 2019.

\bibitem[LeCun(1998)]{lecun1998mnist}
LeCun, Y.
\newblock The {MNIST} database of handwritten digits.
\newblock 1998.
\newblock URL \url{http://yann.lecun.com/exdb/mnist/}.

\bibitem[Littlestone \& Warmuth(1986)Littlestone and
  Warmuth]{littlestone1986relating}
Littlestone, N. and Warmuth, M.
\newblock Relating data compression and learnability.
\newblock 1986.

\bibitem[Madry et~al.(2018)Madry, Makelov, Schmidt, Tsipras, and
  Vladu]{madry2018towards}
Madry, A., Makelov, A., Schmidt, L., Tsipras, D., and Vladu, A.
\newblock Towards deep learning models resistant to adversarial attacks.
\newblock In \emph{6th International Conference on Learning Representations,
  Conference Track Proceedings}. OpenReview.net, 2018.
\newblock URL \url{https://openreview.net/forum?id=rJzIBfZAb}.

\bibitem[Montasser et~al.(2019)Montasser, Hanneke, and Srebro]{montasser2019vc}
Montasser, O., Hanneke, S., and Srebro, N.
\newblock Vc classes are adversarially robustly learnable, but only improperly.
\newblock In \emph{Conference on Learning Theory}, pp.\  2512--2530. PMLR,
  2019.

\bibitem[Montasser et~al.(2021)Montasser, Hanneke, and
  Srebro]{montasser2021transductive}
Montasser, O., Hanneke, S., and Srebro, N.
\newblock Transductive robust learning guarantees.
\newblock \emph{arXiv preprint arXiv:2110.10602}, 2021.

\bibitem[Moosavi-Dezfooli et~al.(2016)Moosavi-Dezfooli, Fawzi, and
  Frossard]{moosavi2016deepfool}
Moosavi-Dezfooli, S.-M., Fawzi, A., and Frossard, P.
\newblock Deepfool: a simple and accurate method to fool deep neural networks.
\newblock In \emph{Proceedings of the IEEE conference on computer vision and
  pattern recognition}, pp.\  2574--2582, 2016.

\bibitem[Moran \& Yehudayoff(2016)Moran and Yehudayoff]{moran2016sample}
Moran, S. and Yehudayoff, A.
\newblock Sample compression schemes for vc classes.
\newblock \emph{Journal of the ACM (JACM)}, 63\penalty0 (3):\penalty0 1--10,
  2016.

\bibitem[Pang et~al.(2022)Pang, Zhang, He, Dong, Su, Chen, Zhu, and
  Liu]{pang2022two}
Pang, T., Zhang, H., He, D., Dong, Y., Su, H., Chen, W., Zhu, J., and Liu,
  T.-Y.
\newblock Two coupled rejection metrics can tell adversarial examples apart.
\newblock In \emph{Proceedings of the IEEE/CVF Conference on Computer Vision
  and Pattern Recognition}, pp.\  15223--15233, 2022.

\bibitem[Peng et~al.(2023)Peng, Xu, Cornelius, Hull, Li, Duggal, Phute, Martin,
  and Chau]{peng2023robust}
Peng, S., Xu, W., Cornelius, C., Hull, M., Li, K., Duggal, R., Phute, M.,
  Martin, J., and Chau, D.~H.
\newblock Robust principles: Architectural design principles for adversarially
  robust cnns, 2023.

\bibitem[Schapire \& Freund(2012)Schapire and Freund]{schapire2012boosting}
Schapire, R.~E. and Freund, Y.
\newblock Boosting. adaptive computation and machine learning.
\newblock \emph{MIT Press, Cambridge, MA}, 1\penalty0 (1.2):\penalty0 9, 2012.

\bibitem[Shalev-Shwartz \& Ben-David(2014)Shalev-Shwartz and
  Ben-David]{shalev2014understanding}
Shalev-Shwartz, S. and Ben-David, S.
\newblock \emph{Understanding machine learning: From theory to algorithms}.
\newblock Cambridge university press, 2014.

\bibitem[Sheikholeslami et~al.(2020)Sheikholeslami, Lotfi, and
  Kolter]{sheikholeslami2020provably}
Sheikholeslami, F., Lotfi, A., and Kolter, J.~Z.
\newblock Provably robust classification of adversarial examples with
  detection.
\newblock In \emph{International Conference on Learning Representations}, 2020.

\bibitem[Sheikholeslami et~al.(2022)Sheikholeslami, Lin, Metzen, Zhang, and
  Kolter]{sheikholeslami2022denoised}
Sheikholeslami, F., Lin, W.-Y., Metzen, J.~H., Zhang, H., and Kolter, J.~Z.
\newblock Denoised smoothing with sample rejection for robustifying pretrained
  classifiers.
\newblock In \emph{Workshop on Trustworthy and Socially Responsible Machine
  Learning, NeurIPS 2022}, 2022.

\bibitem[Sotgiu et~al.(2020)Sotgiu, Demontis, Melis, Biggio, Fumera, Feng, and
  Roli]{sotgiu2020deep}
Sotgiu, A., Demontis, A., Melis, M., Biggio, B., Fumera, G., Feng, X., and
  Roli, F.
\newblock Deep neural rejection against adversarial examples.
\newblock \emph{EURASIP Journal on Information Security}, 2020\penalty0
  (1):\penalty0 1--10, 2020.

\bibitem[Stutz et~al.(2020)Stutz, Hein, and Schiele]{stutz2020confidence}
Stutz, D., Hein, M., and Schiele, B.
\newblock Confidence-calibrated adversarial training: Generalizing to unseen
  attacks.
\newblock In \emph{International Conference on Machine Learning}, pp.\
  9155--9166. PMLR, 2020.

\bibitem[Tram\`er(2022)]{tramer2022detecting}
Tram\`er, F.
\newblock Detecting adversarial examples is (nearly) as hard as classifying
  them.
\newblock In \emph{International Conference on Machine Learning}, pp.\
  21692--21702. PMLR, 2022.

\bibitem[Tramer et~al.(2020)Tramer, Carlini, Brendel, and
  Madry]{tramer2020adaptive}
Tramer, F., Carlini, N., Brendel, W., and Madry, A.
\newblock On adaptive attacks to adversarial example defenses.
\newblock \emph{Advances in Neural Information Processing Systems},
  33:\penalty0 1633--1645, 2020.

\bibitem[Wang et~al.(2021)Wang, Ju, Shelhamer, Wagner, and
  Darrell]{wang2021fighting}
Wang, D., Ju, A., Shelhamer, E., Wagner, D., and Darrell, T.
\newblock Fighting gradients with gradients: Dynamic defenses against
  adversarial attacks.
\newblock \emph{arXiv preprint arXiv:2105.08714}, 2021.

\bibitem[Wang et~al.(2023)Wang, Pang, Du, Lin, Liu, and Yan]{wang2023better}
Wang, Z., Pang, T., Du, C., Lin, M., Liu, W., and Yan, S.
\newblock Better diffusion models further improve adversarial training, 2023.

\bibitem[Wu et~al.(2020)Wu, Yuan, and Wu]{pmlr-v119-wu20f}
Wu, Y.-H., Yuan, C.-H., and Wu, S.-H.
\newblock Adversarial robustness via runtime masking and cleansing.
\newblock In III, H.~D. and Singh, A. (eds.), \emph{Proceedings of the 37th
  International Conference on Machine Learning}, volume 119 of
  \emph{Proceedings of Machine Learning Research}, pp.\  10399--10409. PMLR,
  13--18 Jul 2020.
\newblock URL \url{https://proceedings.mlr.press/v119/wu20f.html}.

\bibitem[Zhang et~al.(2019)Zhang, Yu, Jiao, Xing, El~Ghaoui, and
  Jordan]{zhang2019theoretically}
Zhang, H., Yu, Y., Jiao, J., Xing, E., El~Ghaoui, L., and Jordan, M.
\newblock Theoretically principled trade-off between robustness and accuracy.
\newblock In \emph{International conference on machine learning}, pp.\
  7472--7482. PMLR, 2019.

\bibitem[Zimmermann et~al.(2022)Zimmermann, Brendel, Tramer, and
  Carlini]{binarization}
Zimmermann, R.~S., Brendel, W., Tramer, F., and Carlini, N.
\newblock Increasing confidence in adversarial robustness evaluations.
\newblock In Koyejo, S., Mohamed, S., Agarwal, A., Belgrave, D., Cho, K., and
  Oh, A. (eds.), \emph{Advances in Neural Information Processing Systems},
  volume~35, pp.\  13174--13189. Curran Associates, Inc., 2022.
\newblock URL
  \url{https://proceedings.neurips.cc/paper_files/paper/2022/file/5545d9bcefb7d03d5ad39a905d14fbe3-Paper-Conference.pdf}.

\end{thebibliography}
\end{document}